\documentclass[nohyperref]{article}

\usepackage{microtype}
\usepackage{graphicx}
\usepackage{subfigure}
\usepackage{booktabs}
\usepackage{hyperref}

\PassOptionsToPackage{numbers, compress, sort}{natbib}

\usepackage[final]{neurips_2024}

\usepackage{hyperref}
 \hypersetup{
     colorlinks=true,
     linkcolor=citationcolor,
     filecolor=blue,
     citecolor=citationcolor,      
     urlcolor=cyan,
     }
\usepackage{url}            % simple URL typesetting
\usepackage{booktabs}       % professional-quality tables
\usepackage{amsfonts}       % blackboard math symbols
\usepackage{nicefrac}       % compact symbols for 1/2, etc.
\usepackage{microtype}      % microtypography
\usepackage{graphicx}
\usepackage{subfigure}
\usepackage{booktabs} % for professional tables
\usepackage{amssymb}
\usepackage{amsmath}
\usepackage{amsthm}
\usepackage{mathtools}
\usepackage{tikz}
\usepackage{tikz-cd}
\usepackage{centernot}
\usepackage{bbm}
\usepackage{calc}
\usepackage[numbers]{natbib}

\usetikzlibrary{shapes, decorations}
\usepackage{lipsum}  
\usepackage{algorithm}
\usepackage{algorithmic}
\usepackage{etoolbox}\AtBeginEnvironment{algorithmic}{\small} 

\usepackage{wrapfig, blindtext}
\usepackage{multirow}
\usepackage[outline]{contour}
\usepackage{pifont}
\usepackage[framemethod=tikz]{mdframed}
\usepackage{threeparttable}  

\theoremstyle{plain}
\newtheorem{theorem}{Theorem}[section]

\newtheorem{lemma}[theorem]{Lemma}
\newtheorem{example}[theorem]{Example}
\newtheorem{corollary}[theorem]{Corollary}

\newtheorem{assumption}[theorem]{Assumption}
\theoremstyle{remark}

\usepackage{pifont}
\usepackage{wrapfig, blindtext}
\usepackage{adjustbox}
\usepackage{multirow}
\usepackage{esvect}

\newcommand{\la}{\langle}
\newcommand{\ra}{\rangle}

\newcommand{\mc}{\mathcal}
\newcommand{\mb}{\mathbf}

\newcommand{\norm}[1]{\left\lVert#1\right\rVert}
\usepackage[font=small,labelfont=bf]{caption} % caption font size
\usepackage{etoolbox}\AtBeginEnvironment{algorithmic}{\small} % algorithm font size

\DeclareMathOperator*{\argmin}{arg\,min}
\DeclareMathOperator*{\arginf}{arg\,inf}

\newcommand{\ie}{\textit{i}.\textit{e}., }

\usepackage{multibib}
\newcites{AP}{Appendix References}
\nociteAP{*}

\usepackage{enumitem}

\definecolor{mycolor}{rgb}{1, 1, 1}
\definecolor{mygreen}{RGB}{2,123,2}
\newmdenv[innerlinewidth=0.4pt, roundcorner=5pt,linecolor=mycolor,innerleftmargin=0.5pt,
innerrightmargin=3pt,innertopmargin=0.2pt,innerbottommargin=0.2pt, middlelinewidth = 1pt]{mybox}

\tikzstyle{textbox} = [draw=black,  thick,
    rectangle, rounded corners, inner sep=3pt, inner ysep=1.5pt]
\tikzstyle{title} =[fill=mycolor, text=black, draw=black, rectangle, rounded corners]

\definecolor{citationcolor}{RGB}{80, 90, 180}
\usepackage[textsize=tiny]{todonotes}

\title{Stochastic Optimal Control for Diffusion Bridges in Function Spaces}

\author{
    Byoungwoo~Park$^{1}$,
    Jungwon~Choi$^{1}$,
    Sungbin~Lim$^{2, 3}$\thanks{Equal advising.},
    Juho~Lee$^{1}$\footnotemark[1] \\
    KAIST$^1$, Korea University$^2$, LG AI Research$^3$ \\
    \texttt{\{bw.park,\,jungwon.choi,\,juholee\}@kaist.ac.kr}, \\ \texttt{sungbin@korea.ac.kr}
    }

\begin{document}

\maketitle

\begin{abstract}

Recent advancements in diffusion models and diffusion bridges primarily focus on finite-dimensional spaces, yet many real-world problems necessitate operations in infinite-dimensional function spaces for more natural and interpretable formulations. 
In this paper, we present a theory of stochastic optimal control (SOC) tailored to infinite-dimensional spaces, aiming to extend diffusion-based algorithms to function spaces. 
Specifically, we demonstrate how Doob’s $h$-transform, the fundamental tool for constructing diffusion bridges, can be derived from the SOC perspective and expanded to infinite dimensions. 
This expansion presents a challenge, as infinite-dimensional spaces typically lack closed-form densities. 
Leveraging our theory, we establish that solving the optimal control problem with a specific objective function choice is equivalent to learning diffusion-based generative models. 
We propose two applications: 1) learning bridges between two infinite-dimensional distributions and 2) generative models for sampling from an infinite-dimensional distribution. 
Our approach proves effective for diverse problems involving continuous function space representations, such as resolution-free images, time-series data, and probability density functions. Code is available at
\url{https://github.com/bw-park/DBFS}.

\end{abstract}

\section{Introduction}
Stochastic Optimal Control (SOC) is designed to steer a noisy system toward a desired state by minimizing a specific cost function. This methodology finds extensive applications across various fields in science and engineering, including rate event simulation~\citep{hartmann2019variational, holdijk2022path}, stochastic filtering and data assimilation~\citep{506400, van2007stochastic, reich2019data}, non-convex optimization~\citep{chaudhari2018deep}, modeling population dynamics~\citep{carmona2018probabilistic, liu2022deep}. SOC is also related to diffusion-based sampling methods that are predominant in machine learning literature. Specifically, if we choose the terminal cost of a control problem as the log density ratio between a target distribution and a simple prior distribution, solving the optimal control reduces to learning diffusion-based generative models~\citep{richter2021solving, zhang2022path, vargas2023bayesian} built upon the Schrödinger bridge problem~\citep{Pra1991ASC, chen2020stochastic}. 

While SOC associated diffusion-based generative models have been well-established for finite-dimensional spaces~\citep{berner2022optimal,chen2022likelihood,chen2024generative, liu2024generalized}, their theoretical foundations and practical algorithms for infinite-dimensional spaces remain underexplored. 
There is a growing interest in developing generative models in function spaces. 
Examples include learning neural operators for partial differential equations (PDEs)~\citep{li2020fourier, li2022transformer, kovachki2023neural}, interpreting images as discretized functions through implicit neural representations (INRs)~\citep{sitzmann2020implicit, dupont2021coin}, and operating in function spaces for Bayesian neural networks (BNNs)~\citep{sun2019functional, wang2019function}. Models that operate in function spaces are more parameter-efficient as they avoid resolution-specific parameterizations. 
In response to this demand, several extensions of diffusion-based generative models for infinite-dimensional function spaces have been proposed ~\citep{pidstrigach2023infinite,lim2023score,hagemann2023multilevel,lim2023scorebased, franzese2024continuous, baldassari2024conditional}. 
However, SOC theory for building diffusion bridges in function spaces is still demanding in the community of generative modeling. 

To address these challenges, this work introduces an extension of SOC for diffusion bridges in infinite-dimensional Hilbert spaces, particularly focusing on its applications in sampling problems. 
Specifically, we demonstrate the idea of Doob's $h$-transform~\citep{rogers2000diffusions, chetrite2015nonequilibrium} can be derived from SOC theory and extend its formulation into Hilbert spaces. 
Due to the absence of a density with respect to the Lebesgue measure in infinite-dimensional spaces, building a diffusion bridge between function spaces is a nontrivial task separated from the finite-dimensional cases~\citep{peluchetti2023diffusion, shi2024diffusion}.
To this end, we propose a Radon-Nikodym derivative relative to a specified Gaussian reference measure in Hilbert space. 
Leveraging the infinite-dimensional Doob's $h$-transform and SOC, we then formulate diffusion bridge-based sampling algorithms in function spaces. While the infinite-dimensional Doob's $h$-transform has already been derived in \citep{fuhrman2003class,goldys2008ornstein,baker2024conditioning}, our main goal is not merely to derive it. Instead, we aim to generalize various finite-dimensional sampling problems~\citep{peluchetti2023diffusion, shi2024diffusion, zhang2022path, vargas2023bayesian} into the infinite-dimensional space by exploiting the theory of infinite-dimensional SOC.

To demonstrate the applicability of our theory, we present learning algorithms for two representative problems. 
First, we introduce an infinite-dimensional bridge-matching algorithm as an extension of previous methods~\citep{peluchetti2023diffusion, shi2024diffusion} into Hilbert spaces, which learns a generative model to bridge two distributions defined in function spaces. 
As an example, we show that our framework can learn smooth transitions between two image distributions in a resolution-free manner. 
Second, we propose a simulation-based Bayesian inference algorithm~\citep{zhang2022path, vargas2023bayesian} that operates in function space. Instead of directly approximating the target Bayesian posterior, our algorithm learns a stochastic transition from the prior to the posterior within the function space. We demonstrate the utility of this approach by inferring Bayesian posteriors of stochastic processes, such as Gaussian processes.
We summarize our contributions as follows:
\begin{itemize}[leftmargin=20pt]
    \item Based on the SOC theory, we derive the Doob's $h$-transform in Hilbert spaces. We propose a $h$ function as a Randon-Nikodym derivative with respect to a Gaussian measure in infinite-dimensional space.
    \item Based on the infinite-dimensional extension of the Doob's $h$-transform, we present the diffusion bridge and simulation-based Bayesian inference algorithm in function spaces.
    \item We demonstrate our method for various real-world problems involving function spaces, including resolution-free image translation and posterior sampling for stochastic processes.
\end{itemize}

\paragraph{Notation.}
Consider a real and separable Hilbert space $\mc{H}$ with the norm and inner product denoted by $\norm{\cdot}_{\mc{H}}$ and $\la \cdot, \cdot \ra_{\mc{H}}$.
Throughout the paper, we consider a path measure denoted by $\mathbb{P}^{(\cdot)}$ on the space of all continuous mappings $\Omega=C([0, T], \mc{H})$. The stochastic processes associated with this path measure $\mathbb{P}^{(\cdot)}$ are denoted by $\mb{X}^{(\cdot)}$, and their time-marginal distribution at time $t \in [0, T]$ as push-forward measure $\mu^{(\cdot)}_t := (\mb{X}^{(\cdot)}_t)_{\#}\mathbb{P}^{(\cdot)}$. 
Furthermore, for a function $\mc{V} : [0, T] \times \mc{H} \to \mathbb{R}$, we define $D_{\mb{x}}\mc{V}, D_{\mb{xx}}\mc{V}$ as the first and second order Fréchet derivatives with respect to the variable  $\mb{x} \in \mc{H}$, respectively, and $\partial_t \mc{V}$ as the derivative with respect to the time variable $t \in [0, T]$.

\section{A Foundation on Stochastic Optimal Control for Diffusion Bridges}
In this section, we first present a brief introduction to the theory of stochastic optimal control (SOC) in infinite dimensions and the Verification Theorem (Lemma \ref{lemma:verification}), which is the key to understanding the theoretical connection between SOC and the diffusion bridges. 
Then we propose Doob's $h$-transform in Hilbert spaces based on the SOC theory (Theorems \ref{Theorem:Hopf-Cole} and \ref{Theorem:infinite dimensional density}).

\subsection{Preliminaries}

\paragraph{Gaussian Measure and Cameron-Martin Space.} 
Let $(\Omega, \mc{F}, \mathbb{Q})$ and $(\mc{H}, \mc{B}(\mc{H}))$ be two measurable spaces and consider $\mc{H}$-valued random variable $\mb{X}:\Omega \to \mc{H}$ such that the push-forward measure $\mu := \mb{X}_{\#} \mathbb{Q}$ induced by $\mb{X}$ is Gaussian, i.e., a real-valued random variable $\la u, \mb{X} \ra_{\mc{H}}$ follows Gaussian distribution on $\mathbb{R}$ for any $u \in \mc{H}$. 
Then, there exist a unique mean $m_{\mb{X}} \in \mc{H}$ given by $\la m_{\mb{X}}, u\ra_{\mc{H}} = \mb{E}_{\mu}\left[\la \mb{X}, u \ra_{\mc{H}}\right]$ and a nonnegative, symmetric, and self-adjoint covariance operator $Q : \mc{H} \to \mc{H}$ defined by $\la u, Qv \ra_{\mc{H}} = \mathbb{E}_{\mu}\left[ \la \mb{X} - m_{\mb{X}}, u\ra_{\mc{H}}, \la \mb{X} - m_{\mb{X}}, v\ra_{\mc{H}} \right]$ for any $u, v \in \mc{H}$. 
We write $\mu = \mc{N}(m_{\mb{X}}, Q)$ and say $\mb{X}$ is centered if $m_{\mb{X}} = 0$.
For a covariance operator $Q$, 
we assume $\mc{H}$-valued centered $\mb{X}$ follows the distribution $\mc{N}(0, Q)$ supported on $\mc{H}$ so that $Q$ is guaranteed to be compact. 
Hence, there exists an eigen-system $\{(\lambda^{(k)}, \phi^{(k)}) \in \mathbb{R} \times \mc{H} : k \in \mathbb{N}\}$ such that $Q(\phi^{(k)}) = \lambda^{(k)} \phi^{(k)}$ holds and $\text{Tr}(Q) = \sum_{k=1}^{\infty} \lambda^{(k)} < \infty$. 
We define the \textit{Cameron-Martin space} by $\mc{H}_0 := Q^{1/2}(\mc{H})$, $\mc{H}_0 \subseteq \mc{H}$ is a separable Hilbert space endowed with the inner product $\la u, v \ra_{\mc{H}_0} := \la Q^{-1/2}u, Q^{-1/2}v\ra_{\mc{H}}$.

\paragraph{Stochastic Differential Equations in Hilbert Spaces.}
The standard $\mathbb{R}^d$-valued Wiener process has independent increments $w_{t + \Delta_t} - w_t \sim \mc{N}(0, \Delta_t\mb{I}_d)$. 
In the case of infinite-dimensional Hilbert space $\mc{H}_0$, however, such identity covariance may not be a trace class. 
We consider a larger space $\mc{H}_0 \subset \mc{H}_1$ such that $\mc{H}_0$ is embedded into $\mc{H}_1$ with a Hilbert-Schmidt embedding $J : \mc{H}_0 \to \mc{H}_1$.
Let $Q := J J^{*}$ and we define $\mc{H}_1$-valued $Q$-Wiener process~\citep[Proposition~4.7]{da2014stochastic} as $\mb{W}^Q_t = \sum_{k=1}^{\infty}Q^{1/2} \phi^{(k)} w^{(k)}_t$. 
We focus on the Cameron-Martin space $\mc{H}_0$ where the Wiener process has increments $\mc{N}(0, \Delta_t\mb{I}_{\mc{H}})$, and a larger space $\mc{H}_1 = \mc{H}$, where the Wiener process has increments $\mc{N}(0, \Delta_tQ)$. 
Then we define a path measure $\mathbb{P}$ and associated stochastic differential equation (SDE) in $\mc{H}$ as follows
\begin{equation}\label{eq:uncontrolled SDE}
    d\mb{X}_t = \mc{A} \mb{X}_t + \sigma d\mb{W}^Q_t, \quad \mb{X}_0 \in \mc{H}, \quad t \in [0, T],
\end{equation}
where $\mc{A} : \mc{H} \to \mc{H}$ is a linear operator, a constant $\sigma >0$ and a $Q$-Wiener process $\mb{W}^{Q}$ defined on a probability space $(\Omega, \mc{F}, (\mc{F}_t)_{t \geq 0}, \mathbb{P})$. 
For a more comprehensive understanding, see~\citep{gawarecki2010stochastic, da2014stochastic}.

\subsection{Stochastic Optimal Control in Hilbert Spaces}
From the \textit{uncontrolled} SDE introduced in equation $\eqref{eq:uncontrolled SDE}$, various sampling problems in $\mathbb{R}^d$ including density sampling~\citep{zhang2022path, berner2022optimal, vargas2023bayesian} and generative modeling~\citep{chen2022likelihood, chen2024generative} can be solved by adjusting the SDE with proper drift (control) function $\alpha \in \mathbb{R}^d$. 
Motivated by these approaches, introducing stochastic optimal control (SOC) to solve real-world sampling problems, we aim to introduce SOC to the infinite-dimensional Hilbert space $\mc{H}$.
We consider that a controlled path measures $\mathbb{P}^{\alpha}$ is induced by following infinite-dimensional SDE defined as follows:
\begin{equation}\label{eq:controlled SDE}
    d\mb{X}^{\alpha}_t = \left[ \mc{A}\mb{X}^{\alpha}_t  + \sigma Q^{1/2}\alpha_t \right] dt + \sigma d\mb{\tilde{W}}^Q_t, \quad \mb{X}^{\alpha}_0 = \mb{x}_0, \quad t \in [0, T]
\end{equation}
where $\alpha_{(\cdot)} : [0, T] \times \mc{H} \to \mc{H}$  is infinite-dimensional Markov control (see Section~\ref{subsection:Markov Control} for more details). 
We refer to the SDE in equation~\eqref{eq:controlled SDE} as \textit{controlled} SDE. 
The controlled SDE can be exploited to the various problems. 
In general, it can be done by finding the optimal control function that minimizes the objective functional with suitably chosen cost functionals depending on the problem:
\begin{equation}\label{eq:objective functional}
\mc{J}(t, \mb{x}_{t}, \alpha) = 
\mathbb{E}_{\mathbb{P}^{\alpha}}\left[\int_t^T \left[R(\alpha_s)\right] ds +  G(\mb{X}_T^{\alpha}) \big| \mb{X}^{\alpha}_t = \mb{x}_t  \right],
\end{equation}
where $R : \mc{H} \to \mathbb{R}$ are running cost and $G: \mc{H} \to \mathbb{R}$ is terminal cost. 
The measurable function $\mc{J}(t, \mb{x}, \alpha)$, representing the total cost incurred by the control $\alpha$ over the interval $[t, T]$, given that the control strategy from the interval $[0, t]$ has resulted in $\mb{X}^{\alpha}_t = \mb{x}$. The objective is to minimize the objective functional in \eqref{eq:objective functional} over all admissible control policies $\alpha \in \mc{U}$, where $\mc{U}$ is the Hilbert space of all square-integrable $\mc{H}$-valued processes adapted to $\mb{W}^Q$ defined on $[0, T]$. 
Then we define the \textit{value function} $\mc{V}(t, \mb{x}) = \inf_{\alpha \in \mc{U}} \mc{J}(t, \mb{x}, \alpha)$, the optimal costs conditioned on $(t, \mb{x}) \in [0, T] \times \mc{H}$. 
By using the dynamic programming~\citep{fabbri2017stochastic}, we can solve the Hamilton-Jacobi-Bellman (HJB) equation,
\begin{equation}\label{eq:PDE HJB}
    \partial_t \mc{V}_t + \mc{L} \mc{V}_t + \inf_{\alpha \in \mc{U}} \left[\la \alpha, \sigma Q^{1/2} D_{\mb{x}} \mc{V}_t \ra + R \right] = 0, \quad \mc{V}(T, \mb{x}) = G(\mb{x}),
\end{equation}
where $\mc{L} \mc{V}_t := \la \mb{X}^{\alpha}_t, \mc{A} D_{\mb{x}} \mc{V}_t \ra_{\mc{H}} 
+ \frac{1}{2}\text{Tr}\left[\sigma^2 Q D_{\mb{xx}}\mc{V}_t \right]$ .
We demonstrate that with specific choices of cost functionals $R$ and $G$ in~(\ref{eq:objective functional}), the optimal control $\alpha^{\star}$ of the minimization problem aligns with the proper drift function for sampling problems we will discuss later. To do so, we start with how the HJB equation can characterize the optimal controls.
\begin{lemma}[Verification Theorem]\label{lemma:verification} Let $\mc{V}$ be a solution of HJB equation~\eqref{eq:PDE HJB} with $R(\alpha) := \frac{1}{2}\norm{\alpha}^2_{\mc{H}}$ satisfying the assumptions in~\ref{assumptions value function}. Then, we have $\mc{V}(t, \mb{x}) \leq \mathcal{J}(t, x, \alpha)$ for every $\alpha \in \mathcal{U}$ and $(t, \mb{x}) \in [0, T] \times \mc{H}$. Let $(\alpha^{*}, \mb{X}^{\alpha^{*}})$ be an admissible pair such that
\begin{equation}
\alpha_s^{*} = \arginf_{\alpha \in \mc{U}} \left[\la   \alpha_s, \sigma Q^{1/2} D_{\mb{x}} \mc{V}_t \ra + \frac{1}{2}\norm{\alpha_s}_{\mc{H}}^2 \right] = -\sigma Q^{1/2}D_{\mb{x}}\mc{V}(s, \mb{X}^{\alpha^{*}}_s)
\end{equation}
for almost every $s \in [t,  T]$ and $\mathbb{P}$-almost surely. Then $(\alpha^{*}, \mb{X}^{\alpha^{*}})$ satisfying $\mc{V}(t, \mb{x}) = \mathcal{J}(t, \mb{x}, \alpha^{*})$.
\end{lemma}
Lemma~\ref{lemma:verification} demonstrate that with specific choices of running costs, the solution to the HJB equation in~\eqref{eq:PDE HJB} is the optimal cost of the minimization problem in~\eqref{eq:objective functional} with the closed-form optimal control $\alpha_t^{*} = -\sigma Q^{1/2}D \mc{V}_t$. 
In the subsequent subsection, we reveal the connection between the optimal controlled process, characterized by the controlled SDE with optimal control $\alpha^{\star}$, and the conditioned SDE in $\mc{H}$. 
Additionally, we show various problems depending on the choice of the terminal cost $G$.

\subsection{Doob's  \texorpdfstring{$h$}{h}-transform for Diffusion Bridges in Hilbert Spaces}

The \textit{conditioned} SDE is a stochastic process that is guaranteed to satisfy a set constraint defined over the interval $[0, T]$. 
For instance, a \textit{Diffusion Bridge} is a stochastic process that satisfies both terminal constraints $\mathbf{X}_0 = \mathbf{x}_0$, $\mathbf{X}_T = \mathbf{x}_T$ for any $\mathbf{x}_0, \mathbf{x}_T \in \mathcal{H}$. We will show that with a proper choice of the terminal cost $G$, the conditioned SDE is a special case of controlled SDE. 
For this, let us define the function $h : [0, T] \times \mc{H} \to \mathbb{R}$:
\begin{equation}\label{eq:h function}
    h(t, \mb{x}) = \int_{\mc{H}} \tilde{G}(\mb{z}) \mc{N}_{e^{(T-t)\mc{A}}\mb{x}, Q_{T-t}}(d\mb{z}) = \mathbb{E}_{\mathbb{P}}\left[ \tilde{G}(\mb{X}_T)|\mb{X}_t=\mb{x}\right],
\end{equation}
where $\mc{N}_{e^{t\mc{A}}\mb{x}, Q_{t}}$ is a Gaussian measure with mean function $e^{t\mc{A}}\mb{x}$ and covariance operator $Q_{t} = \int_0^{t} e^{(t-s)\mc{A}} Q e^{(t-s)\mc{A}}ds$ and $\tilde{G} := e^{-G}$ for the function $G$ in~\eqref{eq:objective functional}. 
The function $h$ evaluates the future states $\mb{X}_T$ using $\tilde{G}$ which propagated by~\eqref{eq:uncontrolled SDE} for a given initial $\mb{x}$ at time $t \in [0, T]$. 
It can be shown that the function $h$ satisfies the Kolmogorov-backward equation~\citep{da2014stochastic} with terminal condition,
\begin{equation}\label{eq:PDE KBE}
    \partial_t h_t + \mc{L}h_t=0, \quad h(T, \mb{x}) = \tilde{G}(\mb{x}).
\end{equation}
Now, we employ the Hopf-Cole transformation~\citep{fleming2006controlled} to establish an inherent connection between two classes of PDEs,  the linear PDE in~\eqref{eq:PDE KBE} and the HJB equation in~\eqref{eq:PDE HJB}, which provide us a key insight to deriving the Doob's $h$-transform in function spaces utilizing the SOC theory.
\begin{theorem}[Hopf-Cole Transform]\label{Theorem:Hopf-Cole}  Let $\mc{V}_t = -\log h_t$. Then $\mc{V}_t$ satisfies the HJB equation:
\begin{equation}\label{eq:PDE HJB optimal}
    \partial_t \mc{V}_t + \mc{L} \mc{V}_t -\frac{1}{2}\norm{\sigma Q^{1/2}D_{\mb{x}} \mc{V}_t }^2_{\mc{H}} = 0, \quad \mc{V}(T, \mb{x}) = G(\mb{x}).
\end{equation}
\end{theorem}
According to Theorem~\ref{Theorem:Hopf-Cole}, the solution of linear PDE in equation~(\ref{eq:PDE KBE}) is negative exponential to the solution of the HJB equation in~(\ref{eq:PDE HJB optimal}). Given that it already verified that the optimal control $\alpha^{\star}$ results the value function $\mc{V}_t$, which has explicit form as described in~(\ref{eq:PDE HJB optimal}), we find that the relationship of two PDEs through $\mc{V}_t = -\log h_t$ leads to a distinct form of optimal control $\alpha^{\star} = -\sigma Q^{1/2} D_{\mb{x}} \mathcal{V} = \sigma Q^{1/2} D_{\mb{x}} \log h$, where $D_{\mb{x}} \log h := D_{\mb{x}}h/h$. 
Consequently, it yields another class of SDE as follows:
\begin{equation}\label{eq:conditioned SDE}
    d\mb{X}^h_t = \left[\mc{A}\mb{X}^h_tdt + \sigma^2 Q D_{\mb{x}} \log h(t, \mb{X}_t^h) \right]dt + \sigma d\mb{\hat{W}}_t^{Q}, \quad \mb{X}^h_0 = \mb{x}_0,
\end{equation}
where $\mb{\hat{W}}^Q_t$ is a $Q$-Wiener process on $\mathbb{P}^h$. This representation is consistent with infinite-dimensional conditional SDE~\citep{fuhrman2003class, baker2024conditioning} which induces an expansion of Doob's $h$-transform~\citep{rogers2000diffusions} in Hilbert space $\mc{H}$. 

The Doob's $h$-transform in finite-dimensional spaces is well-established to construct the diffusion bridge process under the assumption that $h(t, \mb{x}) = \mathbb{P}(\mb{X}_T \in d\mb{x}_T |\mb{X}_t=\mb{x})$ has an explicit Radon-Nikodym density function~\citep{liu2023learning, särkkä2019applied}, enable to simulate the bridge SDEs. 
In contrast, although the choice of $\tilde{G}(\mb{x}) = \mb{1}_{d\mb{x}_T}(\mb{x}_T)$ in~\eqref{eq:h function} yields same representation of $h(t, \mb{x}) = \mathbb{P}(\mb{X}_T \in d\mb{x}_T |\mb{X}_t=\mb{x})$, we cannot easily define the Radon-Nikodym density in $\mc{H}$ due to the absence of an equivalent form of Lesbesgue measure which hinder the computation of $D_{\mb{x}} \log h(t, \mb{X}^h_t)$ in~\eqref{eq:conditioned SDE} explicitly. 
Hence, to define the diffusion bridge process in $\mc{H}$, we need to identify an explicit density form of $h(t, \mb{x}) = \mathbb{P}(\mb{X}_T \in d\mb{x}_T |\mb{X}_t=\mb{x})$. 
The following theorem reveals the explicit form of $h$ function and becomes a key ingredient in deriving infinite dimensional diffusion bridge processes.
\begin{theorem}[Explicit Representation of $h$]\label{Theorem:infinite dimensional density}For any $t >0$ and any $\mb{x} \in \mc{H}$, the measure $\mc{N}_{e^{t\mc{A}}\mb{x}, Q_t}$ and $\mc{N}_{0, Q_{\infty}}$ are equivalent, where $\mc{N}_{0, Q_{\infty}}$ is an invariant measure of $\mathbb{P}$ in~\eqref{eq:uncontrolled SDE} as $t \to \infty$ where $Q_{\infty} = -\frac{1}{2}Q \mc{A}^{-1}$. Moreover, for any $\mb{x}, \mb{y} \in \mc{H}$, the Radon-Nikodym density $\frac{d\mc{N}_{e^{t\mc{A}}\mb{x}, Q_t}}{d\mc{N}_{0, Q_{\infty}}}(\cdot) = q_t(\mb{x}, \cdot)$ is given by
\begin{align}
    q_t(\mb{x}, \mb{y}) &=  \text{det}(1-\Theta_t)^{-1/2} \exp \bigg[ - \frac{1}{2} \la (1-\Theta_t)^{-1} Q^{-1/2}_{\infty}e^{t\mc{A}}\mb{x}, Q^{-1/2}_{\infty}e^{t\mc{A}}\mb{x}  \ra_{\mc{H}} \\
    & + \la (1-\Theta_t)^{-1} e^{t\mc{A}}Q^{-1/2}_{\infty}\mb{x}, Q^{-1/2}_{\infty}\mb{y}  \ra_{\mc{H}} - \frac{1}{2} \la \Theta_t (1-\Theta_t)^{-1} Q^{-1/2}_{\infty}\mb{y}, Q^{-1/2}_{\infty}\mb{y}  \ra_{\mc{H}} \bigg],
\end{align}
where $\Theta_t = Q_{\infty}^{1/2} (Q^{-1/2}_t e^{t\mc{A}})^{*}(Q_{\infty}^{-1/2} Q^{1/2}_t)^{*} (Q_{\infty}^{1/2} (Q^{-1/2}_t e^{t\mc{A}})^{*}(Q_{\infty}^{-1/2} Q^{1/2}_t)^{*})^{*}$, $t \geq 0$.
\end{theorem}
Theorem~\ref{Theorem:infinite dimensional density} states that the marginal distribution of certain classes of SDEs described in~\eqref{eq:uncontrolled SDE} has an explicit Radon-Nikodym density with respect to their invariant measure. Therefore, by the time-homogeneity of the process in~\eqref{eq:uncontrolled SDE}, it allows us to define the $h$ function explicitly:
\begin{equation}\label{eq:h function reformulated}
    h(t, \mathbf{x}) = \int_{\mathcal{H}} \tilde{G}(\mathbf{z}) \mathcal{N}_{e^{(T-t)\mathcal{A}}\mathbf{x}, Q_{T-t}}(d\mathbf{z}) = \int_{\mathcal{H}} \tilde{G}(\mathbf{z}) q_{T-t}(\mathbf{x}, \mathbf{z}) \mathcal{N}_{0, Q_{\infty}}(d\mathbf{z}).
\end{equation}
This framework enables the construction of an infinite-dimensional bridge process, by selecting $\tilde{G}$ in~\eqref{eq:h function reformulated} properly. Below, we demonstrate the infinite-dimensional diffusion bridge.
\begin{example}[Diffusion Bridge in $\mc{H}$]\label{example:diffusion bridge} Let $\{(\lambda^{(k)}, \phi^{(k)}) \in \mathbb{R} \times \mc{H} : k \in \mathbb{N}\}$ be an eigen-system of $\mc{H}$. Then for each $k \in \mathbb{N}$, the SDE system in equation~\eqref{eq:uncontrolled SDE} can be represented as:
\begin{align}
    & d\mb{X}^{(k)}_t = -a_k\mb{X}^{(k)}_t dt + \sigma \sqrt{\lambda^{(k)}}d\mb{W}^{(k)}_t, \quad \mb{X}^{(k)}(0) = \mb{x}^{(k)}_0,
\end{align}
where $\mc{A}\phi^{(k)} = -a_k \phi^{(k)}$, $Q \phi^{(k)} = \lambda^{(k)}\phi^{(k)}$, $\mb{X}^{(k)}_t = \la \mb{X}_t, \phi^{(k)}\ra_{\mc{H}}$ and $\mb{W}^{(k)}_t = \la \mb{W}_t, \phi^{(k)}\ra_{\mc{H}}$. Then, for any $\mb{x}_T \in \mc{H}$, the conditional law of $\mb{x}^{(k)}_T$ given $\mb{X}^{(k)}_t$ is a Gaussian  $\mc{N}(\mb{m}^{(k)}_{T|t}\mb{X}_t^{(k)}, \mb{\Sigma}^{(k)}_{T|t})$ with
\begin{equation}
    \mb{m}^{(k)}_{T|t} = e^{-a_k (T-t)}, \quad \mb{\Sigma}^{(k)}_{T|t} = \sigma^2 \frac{\lambda_k}{2a_k}\left(1 - e^{-2a_k(T-t)}\right).
\end{equation}
Now, by setting the terminal condition in~\eqref{eq:h function} as $\tilde{G}(\mb{x}) := \mb{1}_{\mb{x}_T}(\mb{x})$ ($\ie$ Dirac delta of $\mb{x}_T$) then $h(t, \mb{x}) = q_{T-t}(\mb{x}, \mb{x}_T)$. Thus, for each coordinate $k$, we get following representation: 
\begin{align}\label{eq:diffusion bridge process}
    & d\mb{X}^{(k)}_t = \left[-a_k\mb{X}^{(k)}_t + \frac{2a_k e^{-a_k(T-t)}}{1 - e^{-2a_k(T-t)}}(\mb{x}^{(k)}_T - e^{-a_k (T-t)}\mb{X}_t^{(k)})\right]dt + \sigma \sqrt{\lambda^{(k)}}d\mb{W}^{(k)}_t,
\end{align}
with two end points conditions $\mb{X}^{(k)}_0 = \mb{x}^{(k)}_0$ and $\mb{X}^{(k)}_T = \mb{x}^{(k)}_T$.
\end{example}

\subsection{Approximating path measures}\label{sec:Approximating path measures}
Since the function $h$ is intractable for a general terminal cost $\tilde{G}$ in~\eqref{eq:h function reformulated}, simulating the conditioned SDEs in~\eqref{eq:conditioned SDE} requires some approximation techniques. As observed in Theorem~\ref{Theorem:Hopf-Cole}, finding the function $h$ is equal to learning the control function $\alpha$ such that $\mathbb{P}^{\alpha}$ is equal to $\mathbb{P}^{\star} := \mathbb{P}^{\alpha^{\star}} = \mathbb{P}^{h}$. Therefore, with a parametrization $\alpha := \alpha(\cdot, \theta)$ the approximation can be done by neural network parameterization 
$\ie \alpha^{\star} \approx \alpha^{\theta^{\star}}$ with local minimum $\theta^{\star} = \argmin_{\theta} D(\mathbb{P}^{\alpha}||\mathbb{P}^{\star})$, where$D(\mathbb{P}^{\alpha}||\mathbb{P}^{\star})$ is a divergence between $\mathbb{P}^{\alpha}$ and $\mathbb{P}^{\star}$. For example, the cost functional described in equation~\eqref{eq:objective functional} can be represented as relative-entropy loss $D_{\text{rel}}(\mathbb{P}^{\alpha}||\mathbb{P}^{\star}) = \mathbb{E}_{\mathbb{P}^{\alpha}}\left[\log \frac{d\mathbb{P}^{\alpha}}{d\mathbb{P}^{\star}}\right]$\footnote{See Sec~\ref{sec:Deriving Divergence between Path Measures_appx} for more details}. Therefore, the training loss for $\theta$ can be estimated by first simulating the parameterized control path and then calculating \eqref{eq:objective functional} for a specified cost functional $R, G$. Moreover, if we can access to the $\mathbb{P}^{\star}$, one can define the variational optimization~\citep{8618948} where the loss is defined as cross-entropy loss $D_{\text{cross}}(\mathbb{P}^{\alpha}||\mathbb{P}^{\star}) = \mathbb{E}_{\mathbb{P}^{\star}}\left[\log \frac{d\mathbb{P}^{\star}}{d\mathbb{P}^{\alpha}}\right]$. 
See~\citep{nusken2021solving, domingo2023stochastic} for more details about the approximation technique and other loss functions.

\section{Simulating Diffusion Bridges in Infinite Dimensional Spaces}
Leveraging the SOC theory within $\mc{H}$, we show how our approach generalizes existing diffusion-based sampling methods. Specifically, incorporating the relation between \textit{controlled} SDEs~\eqref{eq:controlled SDE} and \textit{conditioned} SDEs~\eqref{eq:conditioned SDE}, we introduce two learning algorithms that allow us to simulate various diffusion bridge-based sampling algorithms.

\subsection{Infinite Dimensional Bridge Matching}\label{subsection:bridge matching}
In this section, our objective is to learn a control $\alpha$ that yields $\mathbb{P}^{\alpha}$ such that $\{\mb{X}^{\alpha}_t\}_{t \in [0, T]}$ satisfies $\mu^{\alpha}_t \approx \mu^{\star}_t$ for all pre-specified $\mu^{\star}_t$ over the interval $t \in [0, T]$. Specifically, we assume that the end-point marginals $\mu^{\star}_0$ and $\mu^{\star}_T$ follow the laws of two data distributions $\pi_0$ and $\pi_T$, respectively, and the intermediate marginals $\{\mu^{\star}_t\}_{t \in (0, T)}$ are defined as a mixture of diffusion bridge paths. This learning problem is referred to as the \textit{Bridge Matching} (BM) algorithm~\citep{peluchetti2022nondenoising, shi2024diffusion, liu2023i2sb} and can be expressed as a solution of the SOC problem structured as
\begin{equation}\label{eq:bridge matching problem}
\inf_{\alpha}D(\mathbb{P}^{\alpha}|\mathbb{P}^{\star}), \; \text{such that} \;  d\mb{X}^{\alpha}_t = \left[ \mc{A}\mb{X}^{\alpha}_t  + \sigma Q^{1/2}\alpha_t \right] dt + \sigma d\mb{\tilde{W}}^Q_t, \quad \mb{X}^{\alpha}_0 \sim \pi_0.
\end{equation}
In~\eqref{eq:bridge matching problem}, various divergences can be chosen for the same learning problem, as discussed in Section~\ref{sec:Approximating path measures}. Here, we will choose the cross-entropy because the relative entropy requires the appropriate selection of the terminal cost $G$ in~\eqref{eq:objective functional}, which is intractable since we do not have access to the distributional form of $\pi_0, \pi_T$~\citep{liu2024generalized}.
Furthermore, keeping the entire computational graph of $\mathbb{P}^{\alpha}$ with parameterized $\alpha$ can become resource-intensive, especially for higher-dimensional datasets like images~\citep{chen2022likelihood}.

Now, we specify the optimal path measure $\mathbb{P}^{\star}$ for a problem in~\eqref{eq:bridge matching problem}.
Let $\mathbb{P}_{|0, T}$ be a path measure induced by~\eqref{eq:diffusion bridge process} and $\mu_{t|0, T}$ be a marginal distribution of $\mathbb{P}_{|0, T}$. Moreover, let $\mathbb{P}^{\star} = \mathbb{P}_{|0, T} \Pi_{0, T}$ for an independent coupling $\Pi_{0, T} = \pi_0 \otimes \pi_T$. Then the optimal path measure $\mathbb{P}^{\star}$ is defined as \textit{Mixture of bridge}. 
Under regular assumptions, the optimal control $\alpha^{\star}$ that induces the optimal path measure $\mathbb{P}^{\star}$ can be constructed as a mixture of functions $h$ in~\eqref{eq:conditioned SDE} by choosing $G(\mb{x}) = \mb{1}_{\mb{x}_T}(\mb{x})$.
\begin{theorem}[Mixture of Bridges in $\mc{H}$]\label{theorem:mixture of bridge} Let us consider a marginal distribution of $\mathbb{P}^{\star}$ at $t \in [0, T]$, $\mu^{\star}_t(d\mb{x}_t) = \int \mu_{t|0, T}(d\mb{x}_t) \Pi_{0, T}(d\mb{x}_0, d\mb{x}_T)$ has density $p^{\star}_t$ with respect to some Gaussian reference measure $\mu_{\text{ref}}$ $\ie \mu^{\star}_t(d\mb{x}_t)/\mu_{\text{ref}}(d\mb{x}_t) = p^{\star}_t(\mb{x}_t)$. Then the optimal path measure $\mathbb{P}^{\star}$ associated with:
\begin{equation}\label{eq:target SDE}
    d\mb{X}^{\star}_t = \left[\mc{A}\mb{X}^{\star}_t + \mathbb{E}_{\mb{x}_T \sim \mathbb{P}^{\star}(d\mb{x}_T|\mb{X}^{\star}_t)}\left[ \sigma^2 Q D_{\mb{x}} \log \mc{N}(\mb{x}_T; \mb{m}_{T|t} \mb{X}^{\star}_t, \mb{\Sigma}_{T|t}) \right]\right]dt + \sigma d\mb{\hat{W}}^Q_t,
\end{equation}
where $\mb{m}_{T|t} = e^{(T-t)\mc{A}}$ and $\mb{\Sigma}_{T|t} = \sigma^2 \int_0^{T-t} e^{(T-t-s) \mc{A}} Q e^{(T-t-s) \mc{A}} ds$ and $\mb{X}^{\star}_t \sim \mu_t$ for $t \in [0, T]$.
\end{theorem}
\paragraph{Objective Functional for Bridge Matching.}
With the structure of $\mathbb{P}^{\star}$ specified in~\eqref{eq:target SDE} we can estimate the divergence between the optimal target path measure $\mathbb{P}^{\star}$ and a path measure $\mathbb{P}^{\alpha}$.
To accomplish this, we first define
\begin{equation}\label{eq:cross entropy gamma}
    \gamma(t, \mb{x}; \theta) = Q^{1/2} \left[ \mathbb{E}_{\mb{x}_T \sim \mathbb{P}^{\star}(d\mb{x}_T|\mb{x})}\left[ \sigma Q^{1/2} D_{\mb{x}} \log \mc{N}(\mb{x}_T; \mb{m}_{T|t} \mb{X}^{\star}_t, \mb{\Sigma}_{T|t}) \right] - \alpha(t, \mb{x}; \theta)\right].
\end{equation}
Then, by applying the Girsanov theorem\footnote{See Sec~\ref{sec:Deriving Divergence between Path Measures_appx} for details.} which provides us the Radon-Nikodym derivative between $\mathbb {P}^{\star}$ and $\mathbb{P}^{\alpha}$, we can derive the cross-entropy loss in equation~\eqref{eq:bridge matching problem}:
\begin{equation}\label{eq:cross entropy loss}
    D_{\text{cross}}(\mathbb{P}^{\alpha^{\theta}}|\mathbb{P}^{\star}) = \mathbb{E}_{\mathbb{P}^{\star}}\left[ \log \frac{d\mathbb{P}^{\star}}{d\mathbb{P}^{\alpha^{\theta}}}\right] = \mathbb{E}_{\mathbb{P}^{\star}}\left[ \int_0^T \frac{1}{2}\norm{\gamma(t, \mb{X}^{\star}_t ; \theta)}^2_{\mc{H}_0} ds \right]
\end{equation}
Then, under the neural network parameterization of control function $\alpha^{\theta}$, we can reformulate the SOC problem in~\eqref{eq:bridge matching problem} as a learning problem with the training loss function represented by:
\begin{equation}\label{eq:training loss bridge matching}
    \mc{L}_{\text{BM}}(\theta) = \mathbb{E}_{t \sim \mc{U}_{0, T}} \mathbb{E}_{\mathbb{P}^{\star}(\mb{x}_T \sim d\mb{x}_T|\mb{X}^{\star}_t)}\left[ \frac{1}{2}\norm{\sigma Q^{1/2} D_{\mb{x}} \log \mc{N}(\mb{x}_T; \mb{m}_{T|t} \mb{X}^{\star}_t, \mb{\Sigma}_{T|t}) - \alpha(t, \mb{X}^{\star}_t; \theta)}^2_{\mc{H}}\right]
\end{equation}
The $\mc{L}_{\text{BM}}(\theta)$ in~\eqref{eq:training loss bridge matching} yields the infinite-dimensional BM summarized in Alg~\ref{algorithm:BM}.

\begin{figure}[!t]
\begin{minipage}[t]{0.47\textwidth}
  \begin{algorithm}[H]
    \caption{Bridge Matching of DBHS}
    \begin{algorithmic}\label{algorithm:BM}
    \STATE \textbf{Input:} Coupling $\Pi_{0, 1}$, Bridge $\mathbb{P}_{|0, T}$
      \FOR{$n=1, \cdots, N$}
        \STATE Sample $(\mb{x}_0, \mb{x}_T) \sim \Pi_{0, T}$
        \STATE Simulate $\mb{X}_{[0,T]}^{\star} \sim \mathbb{P}_{|0,T}$ with $(\mb{x}_0, \mb{x}_T)$
        \STATE Compute $\mc{L}_{\text{BM}}(\theta_k)$ wtih~\eqref{eq:training loss bridge matching}
        \STATE Update $\theta_{n+1}$ with $\nabla_{\theta_n} \mc{L}_{\text{BM}}(\theta_n)$
      \ENDFOR
    \STATE \textbf{Output:} Approximated optimal control $\alpha^{\theta^{\star}}$
    \end{algorithmic}
  \end{algorithm}
\end{minipage}
\begin{minipage}[t]{0.51\textwidth}
  \begin{algorithm}[H]
    \caption{Bayesian Learning in $\mc{H}$}
    \begin{algorithmic}\label{algorithm:exact sampling}
    \STATE \textbf{Input:} 
    Initial condition $\mb{x}_0$, energy functional $\mc{U}$
      \FOR{$n=1, \cdots, N$}
        \STATE Simulate $\mb{X}^{\alpha^{\theta}}_{[0, T]} \sim \mathbb{P}^{\alpha^{\theta}}$ with $\mb{X}^{\alpha^{\theta}}_0 = \mb{x}_0$
        \STATE Compute $\mc{L}_{\text{Bayes}}(\theta_k)$ with~\eqref{eq:relative entropy loss}
        \STATE Update $\theta_{n+1}$ with $\nabla_{\theta_n} \mc{L}_{\text{Bayes}}(\theta_n)$
      \ENDFOR
    \STATE \textbf{Output:} Approximated optimal control $\alpha^{\theta^{\star}}$
    \end{algorithmic}
  \end{algorithm}
\end{minipage}
\vspace{-5mm}
\end{figure}
\subsection{Bayesian Learning in Function Space}

In the previous section, we observed that by appropriately defining a terminal cost functional $G$ in~\eqref{eq:objective functional}, the SOC problem aligns with the sampling problem, where optimal control effectively steers the distribution from $\pi_{0}$ to the target distribution $\pi_{T}$, where we can access samples from $\pi_0$ and $\pi_T$. However, accessing samples from unknown target distribution $\pi_T$ is generally not feasible. For instance, for a $\pi_T$, a posterior distribution over function. In this case, although direct samples from $\pi_T$ are unattainable, its distributional representation is given as~\citep{suzuki2020generalization, baldassari2024conditional}:
\begin{equation}\label{eq:posterior measure}
    \frac{d\pi_T}{d\mu_{\text{prior}}}(\mb{X}_T) \propto \exp\left(-\mc{U}(\mb{X}_T)\right), \quad \mu_{\text{prior}} = \mc{N}(\mb{m}_{\text{prior}}, Q_{\text{prior}}),
\end{equation}
where $\mc{U}$ is a energy function. Here, our primary objective is to sample from a distribution over function $\pi_T := \mu^{\star}_T$ by simulating the controlled diffusion process $\{\mb{X}_t^{\alpha}\}_{t \in [0, T]}$ over finite horizon $[0, T]$ with $T < \infty$. It can be represented as a solution of the following SOC problem:
\begin{equation}\label{eq:bayesian learning problem}
\inf_{\alpha}D(\mathbb{P}^{\alpha}|\mathbb{P}^{\star}), \; \text{such that} \;  d\mb{X}^{\alpha}_t = \left[ \mc{A}\mb{X}^{\alpha}_t  + \sigma Q^{1/2}\alpha_t \right] dt + \sigma d\mb{\tilde{W}}^Q_t, \quad \mb{X}^{\alpha}_0 = \mb{x}_0.
\end{equation}
The following theorem implies that with a suitable terminal cost functional $G$ in~\eqref{eq:objective functional}, it is possible to achieve $\mb{X}^{\alpha^{\star}}_T \sim \pi_T$ as an expansion of~\citep{Pra1991ASC,tzen2019theoretical} for infinite dimensional space $\mc{H}$.
\begin{theorem}[Exact sampling in $\mc{H}$]\label{theorem:exact sampling} Consider that the initial distribution $\mu_0$ is given as the Dirac measure $\delta_{\mb{x}_0}$ for some $\mb{x}_0 \in \mc{H}$ and the following objective functional
\begin{equation}\label{eq:objective function exact sampling}
    \mc{J}(\alpha) = \mathbb{E}_{\mathbb{P}^{\alpha}}\left[ \int_0^T \frac{1}{2}\norm{\alpha_s}_{\mc{H}}^2 ds - \log \frac{d\pi_T}{d\mu_T}(\mb{X}^{\alpha}_T) \right]
\end{equation}
where $\mu_T = \mc{N}(e^{T\mc{A}}\mb{x}_0, Q_T)$ as a marginal distribution of $\mb{X}_T$ in~\eqref{eq:uncontrolled SDE} with a well-defined terminal cost $\frac{d\pi_T}{d\mu_T}$ by Theorem~\ref{Theorem:infinite dimensional density}. Then, $\mb{X}^{\alpha^{\star}}_T \sim \pi_T$.
\end{theorem}
\paragraph{Objective Functional for Bayesian Learning.} 
Unlike problem in~\eqref{eq:bridge matching problem} where we can access to the target path measure $\mathbb{P}^{\star}$ directly, it is not feasible here because $h(t, \mb{x}) = \mathbb{E}_{\mathbb{P}}\left[-\frac{d\pi_T}{d\mu_T}(\mb{X}_T) | \mb{X}_t=\mb{x}\right]$ does not have an explicit solution. Therefore, we will use the relative-entropy loss as our training loss function:
\begin{equation}\label{eq:relative entropy loss}
    \mc{L}_{\text{Bayes}}(\theta) = D_{\text{rel}}(\mathbb{P}^{\alpha^{\theta}}|\mathbb{P}^{\star}) = \mathbb{E}_{\mathbb{P}^{\alpha^{\theta}}}\left[ \int_0^T \frac{1}{2}\norm{\alpha(s, \mb{X}^{\alpha^{\theta}}_s; \theta)}_{\mc{H}}^2 ds - \log \frac{d\pi_T}{d\mu_T}(\mb{X}^{\alpha^{\theta}}_T) \right].
\end{equation}
The key difference from previous algorithms~\citep{zhang2022path, vargas2023bayesian} is that the Radon-Nikodym derivative $\frac{d\pi_T}{d\mu_T}(\mb{X}^{\alpha}_T)$ may not be well-defined on $\mc{H}$ due to the absence of the Lesbesgue measure. However, the Theorem~\ref{Theorem:infinite dimensional density} suggests that by choosing certain classes of Gaussian measure $\ie \mu_{\text{prior}} := \mc{N}(e^{t\mc{A}}\mb{x}, Q_t)$, the Radon-Nikodym density of $\frac{d\pi_T}{d\mu_{\text{prior}}}$ and $\frac{d\mu_T}{d\mu_{\text{prior}}}$ has explicit form since $\mu_{\text{prior}}$ and $\mc{N}(0, Q_{\infty})$ are equivalent.  Thus, using the chain rule, the terminal cost for any $\mb{x} \in \mc{H}$ can be computed as follows:
\begin{equation}
    \log \frac{d\pi_T}{d\mu_T}(\mb{x}) = -\mc{U}(\mb{x}) - \log \frac{d\mu_{\text{prior}}}{d\mc{N}(0, Q_{\infty})}(\mb{x}) + \log \frac{d\mu_T}{d\mc{N}(0, Q_{\infty})}(\mb{x}).
\end{equation}
With $\mc{L}_{\text{Bayes}}(\theta)$ in~\eqref{eq:relative entropy loss}, the infinite-dimensional bayesian learning algorithm is summarized in Alg~\ref{algorithm:exact sampling}.

% \begin{figure}
% \begin{minipage}[b]{0.245\linewidth} 
% \includegraphics[width=1.\textwidth,]{Figure/128_0.png}
% \end{minipage}
% \hspace{-2mm}
% \begin{minipage}[b]{0.245\linewidth} 
% \includegraphics[width=1.\textwidth,]{Figure/128_1.png}
% \end{minipage}
% \hspace{-0.7mm}
% \begin{minipage}[b]{0.49\linewidth} 
% \includegraphics[width=1.\textwidth,]{Figure/256_1.png}
% \end{minipage}
% \caption{Samples of \textit{zero-shot transfer} from dog to cat. By modeling the diffusion process in function space $\mc{H}$, DBFS is capable of generating realistic samples on unseen, finer grids: (Left) $128^2$ and (Right) $256^2$. This is achieved despite the control being trained solely on a $64^2$ grid. Note that all the dog images in our experiment on $128^2$ and $256^2$ grids were initially upsampled from $64^2$.}\label{fig:resolution_128_256}

% \begin{minipage}[b]{0.49\linewidth} 
% \includegraphics[width=1.\textwidth,]{Figure/64_1.png}
% \end{minipage}
% \begin{minipage}[b]{0.49\linewidth} 
% \includegraphics[width=1.\textwidth,]{Figure/64_0.png}
% \end{minipage}
% \caption{Sample trajectories $\mb{X}^{\alpha}_t$ on trained grid $64^2$. DBFS diffuses the data across the function space.}\label{fig:resolution_64}
% \vspace{-6mm}
% \end{figure}

\section{Related Work}

Most diffusion models operate within the framework of time-reversal~\citep{ANDERSON1982313}, where the generation process is learned from its corresponding time-reversed SDEs~\citep{song2021scorebased}. In contrast, diffusion models based on conditioned SDEs, such as diffusion bridges, built upon the theory of Doob's $h$-transform, offer a conceptually simpler approach as they solely rely on a forward process. \citep{peluchetti2022nondenoising} proposes generative models with this concept, showing that the mixture of forward diffusion bridge processes effectively transports between couplings of two distributions. \citep{ye2022first} introduces a family of first hitting diffusion models that generate data with a forward diffusion process at a random first hitting time based on Doob's $h$-transform. Combining time-reversal with the $h$-transform, \citep{liu2023learning} proposes a diffusion bridge process on constrained domains. Moreover, \citep{shi2024diffusion, peluchetti2023diffusion} presented that the Schrödinger bridge problem can be solved by an iterative algorithm, which is improved by~\citep{debortoli2024schr} to enhance efficiency. Furthermore, \citep{liu2024generalized} generalizes the Schrödinger bridge matching algorithm by introducing an approximation scheme with a non-trivial running cost. Compared to prior works, which primarily focus on finite-dimensional spaces, our work extends the formulation of Doob's $h$-transform into Hilbert space, enabling the development of various sampling algorithms in function spaces.
\section{Experiments}

This section details the experimental setup and the application of the proposed Diffusion Bridges in Function Spaces \textbf{(DBFS)} for generating functional data. We interpret the data from a functional perspective, known as \textit{field} representation~\citep{xie2022neural,zhuang2023diffusion}, where data are seen as a finite collection of function evaluations $\{\mb{Y}[\mb{p}_i], \mb{p}_i\}_i^{N}$. Here, a function $\mb{Y}$ maps points $\mb{p}_i$ from a coordinate space $\mc{X}$ to a signal space $\mc{Y}$, i.e., $\mb{Y} : \mc{X} \to \mc{Y}$. Additional experimental details are provided in Appendix~\ref{sec:experimental details}.

\subsection{Bridge Matching}
First, we present empirical results for the infinite-dimensional BM algorithm discussed in Sec~\ref{subsection:bridge matching}, applied to 1D and 2D data. For 1D data, we consider $\mc{X} = \mathbb{R}$ and $\mc{Y} = \mathbb{R}$. For 2D data, we assume $\mc{X} = \mathbb{R}^2$ and $\mc{Y} = \mathbb{R}$ for probability density or grayscale images, and $\mc{Y} = \mathbb{R}^3$ for RGB images.
\paragraph{Bridging Field.}
\begin{wrapfigure}{r}{0.5\textwidth}
\vspace{-2mm}
\begin{minipage}[b]{0.24\linewidth}
    \centering
    \scriptsize
    {$\mb{X}_0 = p_0$}
\end{minipage}
\begin{minipage}[b]{0.5\linewidth}
    \centering
    \scriptsize
    {$\mb{X}_t \sim \mu_{t|0,T}$}
\end{minipage}
\begin{minipage}[b]{0.24\linewidth}
    \centering
    \scriptsize
    {$\mb{X}_T = p_T$}
\end{minipage}

\centering
\begin{minipage}[b]{1.\linewidth}
\includegraphics[width=0.24\textwidth,]{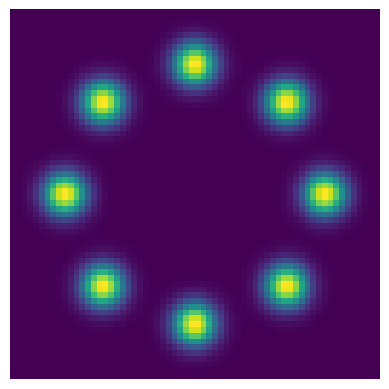}
\includegraphics[width=0.24\textwidth,]{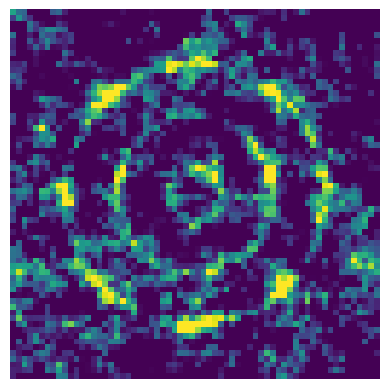}
\includegraphics[width=0.24\textwidth,]{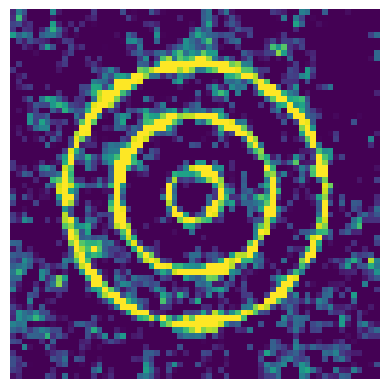}
\includegraphics[width=0.24\textwidth,]{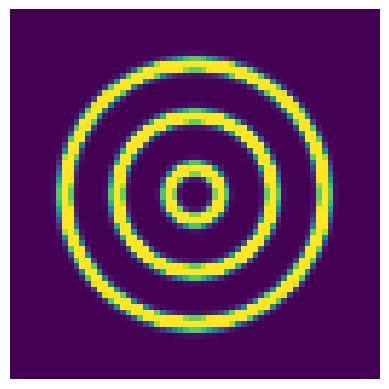}
\end{minipage}

\begin{minipage}[b]{1.\linewidth}
\includegraphics[width=0.24\textwidth,]{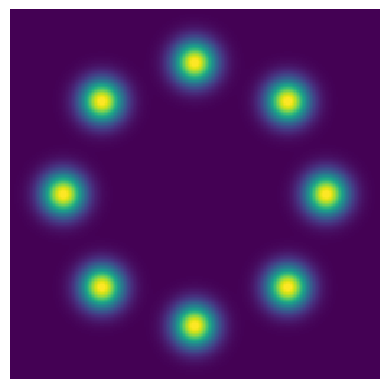}
\includegraphics[width=0.24\textwidth,]{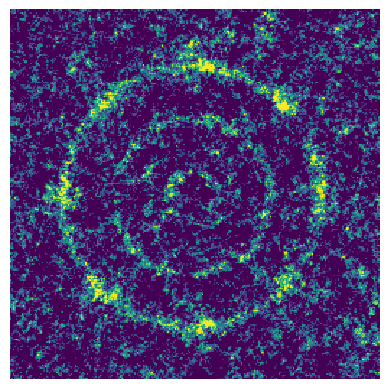}
\includegraphics[width=0.24\textwidth,]{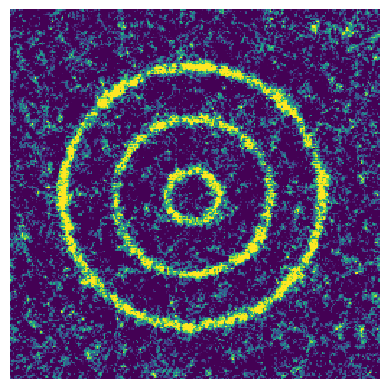}
\includegraphics[width=0.24\textwidth,]{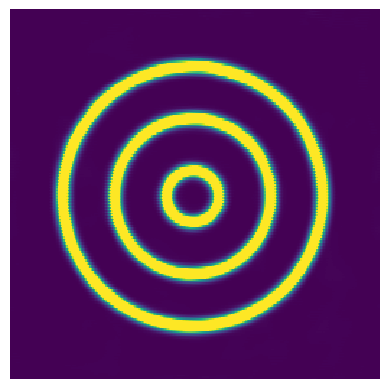}
\end{minipage}
\caption{(Top) Diffusion Bridge $\mathbb{P}^{\star}$ evaluated on $32^2$. (Bottom) Learned  process $\mathbb{P}^{\alpha^{\star}}$  evaluated on $256^2$}\label{figure:bridge matching dirac}
\vspace{-4mm}
\end{wrapfigure}
We begin by validating our bridge matching Algorithm in Alg~\ref{algorithm:BM} on bridging probability density function within $\mc{H}$. Specifically, we set $\pi_0 := \delta_{p_0}$ with a ring-shaped density function $p_0$ and $\pi_T := \delta_{p_T}$ characterized by a Gaussian mixture density function $p_T$. The functions map each grid points $\mb{p}_i$ to the probability in $\mc{Y}=\mathbb{R}$. Therefore, both density functions can be represented as their field representations $\{p_0[\mb{p}_i], \mb{p}_i\}_{i}^N, \{p_T[\mb{p}_i], \mb{p}_i\}_{i}^N$, respectively. Figure~\ref{figure:bridge matching dirac} illustrates the progressive propagation of the target optimal bridge process $\mathbb{P}^{\star}$ from $p_0$ to $p_T$. Despite the $\alpha^{\star}$ is trained on the functions generated from $\mathbb{P}^{\star}$ which are evaluated on a coarse grid $\{\mb{p}_i\}_{i}^{32^2}$, $\mathbb{P}^{\alpha^\star}$ is capable of producing accurate functional evaluations on a finer grid $\{\mb{p}_i\}_{i}^{256^2}$. This resolution-invariance property indicates that our method is adept at learning continuous functional representations, rather than merely memorizing the discrete evaluations.

\begin{figure}[!h]
  \centering
  % \begin{minipage}[t]{.49\textwidth}
    \centering
    \subfigure[Quadratic]{
        \includegraphics[width=0.31\textwidth]{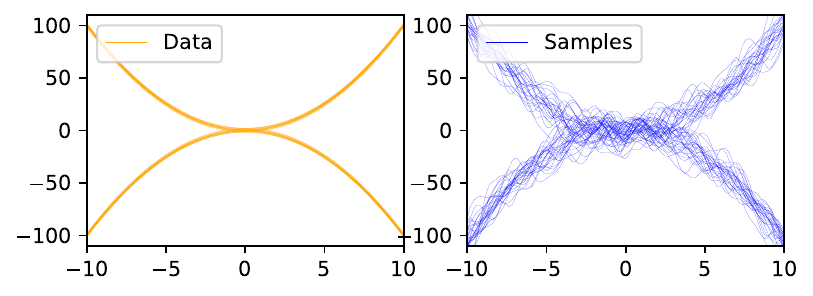}
    }
    \subfigure[Melbourne]{
        \includegraphics[width=0.31\textwidth]{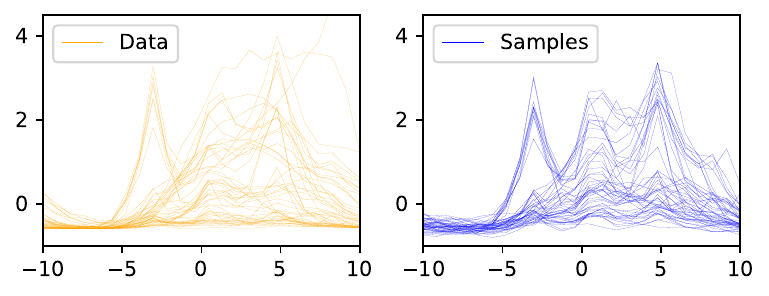}
    }
    \subfigure[Gridwatch]{
        \includegraphics[width=0.31\textwidth]{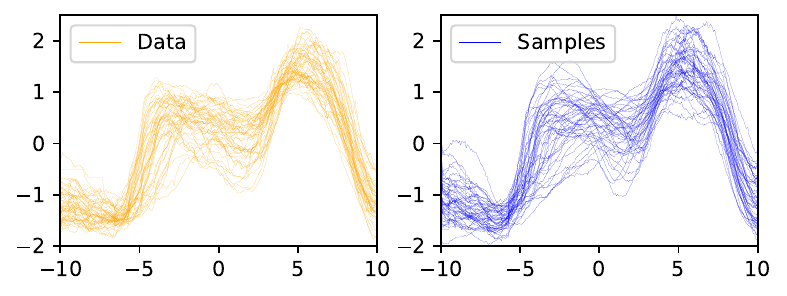}
    }
    \caption{Results on 1D function generation. (Left) Real data and (Right) generated samples from our model.}\label{figure:1d_generation}
\vspace{-4mm}
\end{figure}

\begin{wraptable}[5]{r}{0.5\textwidth}
    \vspace{-5mm}
    \centering
    \scriptsize
    \captionof{table}{A Power$(\%)$ of a kernel two-sample test.}
    \begin{tabular}{l|c|c|c}
    \hline
     & \textbf{NDP}~\citep{dutordoir2023neural} & \textbf{SP-SGM}~\citep{phillips2022spectral} & \textbf{DBFS}  (Ours) \\
    \hline
    Quadratic & $\geq$ 99.0 & 5.4 $\pm$ 0.7 & \textbf{5.1 $\pm$ 0.4} \\
    Melbourne & 12.8 $\pm$ 0.4 & \textbf{5.3 $\pm$ 0.7} & 9.67 $\pm$ 0.45 \\
    Gridwatch & 16.3 $\pm$ 1.8 & 4.7 $\pm$ 0.5 & \textbf{3.9 $\pm$ 0.4} \\
    \hline
    \end{tabular}\label{tab:1d_generation}
\end{wraptable}

\paragraph{1D function generation. }
We conducted an experiment on a 1D function generation task, comparing our baseline methods~\citep{dutordoir2023neural, phillips2022spectral} on three datasets: Quadratic, Melbourne, and Gridwatch, following the setup from~\citep{phillips2022spectral}. For generative modeling, we set the initial distribution $\pi_0$ as $\mc{N}(0, Q)$ with RBF kernel for the covariance operator $Q$ and 
the terminal distribution as data distribution $\pi_T$, respectively. We employing the bridge matching algorithm in Alg~\ref{algorithm:BM}. For quantitative evaluation, we used the power of a kernel two-sample hypothesis test to distinguish between generated and ground-truth samples. Table~\ref{tab:1d_generation} shows that our method performs comparably to baseline infinite-dimensional methods. 
Additionaly, The generated samples compared to the ground-truth for each dataset are provided in Figure~\ref{figure:1d_generation}.

\begin{figure}[!t]
  \centering
  % \begin{minipage}[t]{.49\textwidth}
    \centering
    \subfigure[$32^2$ (observed resolution)]{
        \includegraphics[width=0.315\textwidth]{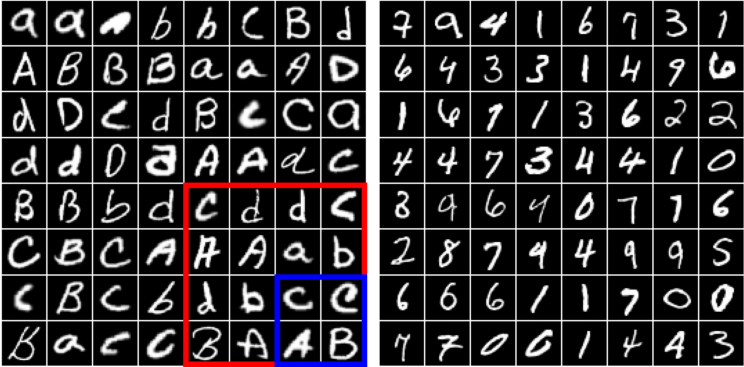}
    }
    \subfigure[$64^2$ (unseen resolution)]{
        \includegraphics[width=0.315\textwidth]{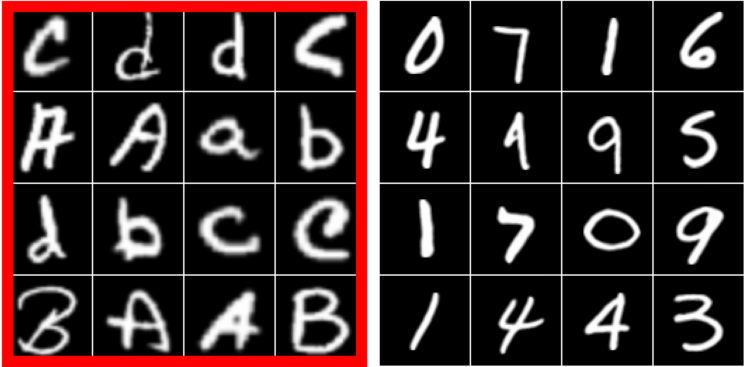}
    }
    \subfigure[$128^2$ (unseen resolution)]{
        \includegraphics[width=0.315\textwidth]{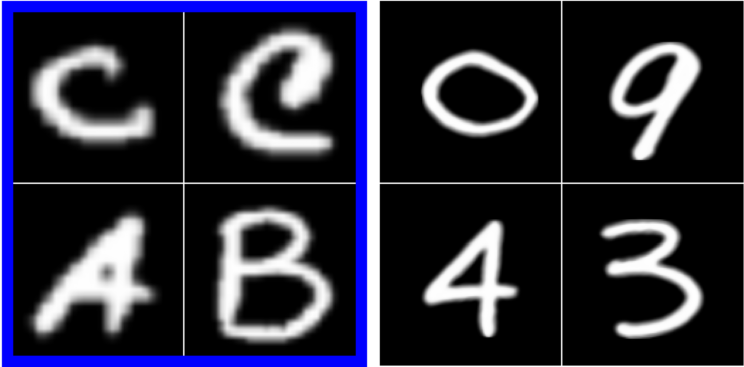}
    }
    \subfigure[$64^2$ (observed resolution)]{
        \includegraphics[width=0.475\textwidth]{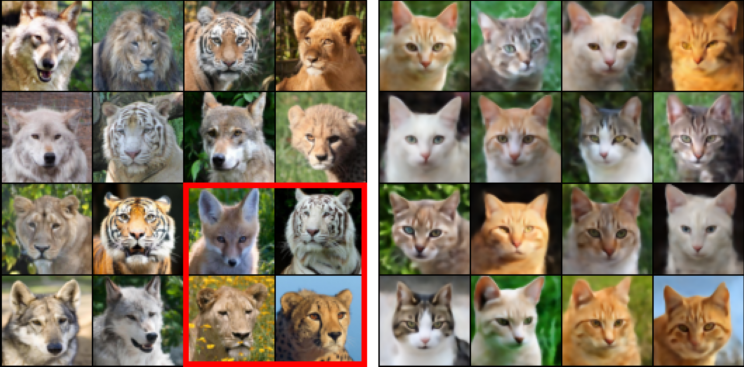}
    }
    \subfigure[$128^2$ (unseen resolution)]{
        \includegraphics[width=0.475\textwidth]{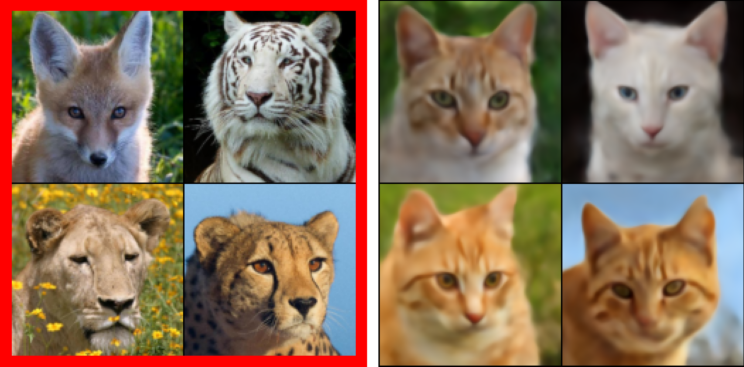}
    }
    \caption{Results on Unpaired image transfer task. \textbf{(Up)} EMNIST $\to$ MNIST \textbf{(Down)} AFHQ-64 Wild $\to$ Cat. (Left) Real data and (Right) generated samples from our model. For generation at unseen resolutions, the images within the \textcolor{red}{red} and \textcolor{blue}{blue} boxed initial conditions were upsampled (using bi-linear transformation) from the observed resolution ($32^2$) for EMNIST and ($64^2$) for AFHQ-64 Wild, respectively.}\label{figure:2d_generatior_afhq}
  % \end{minipage}
\vspace{-5mm}
\end{figure}

\paragraph{Unpaired Image Transfer. }\label{sec:bridge matching experiment}

We compare our proposed model with a finite (fixed)-dimensional baseline through an experiment on unpaired image transfer between the MNIST and EMNIST datasets at $32^2$ resolution, as well as wild and cat images from the AFHQ dataset~\citep{choi2020starganv2}, down-sampled to $64^2$ resolution (AFHQ-64). Specifically, we evaluate the performance of~\citep{peluchetti2023diffusion, shi2024diffusion} and our DBFS model. For a fair comparison, we follow the iterative training scheme of~\citep{peluchetti2023diffusion} based on the public repository\footnote{\url{https://github.com/stepelu/idbm-pytorch}, under MIT License.}, where two forward and backward control networks are trained alternately. For quantitative evaluation, we estimate the FID score between the generated samples and real datasets. We set $\sigma = 1$ for both~\citep{peluchetti2023diffusion} and our method, while FID scores for~\citep{shi2024diffusion} are taken from~\citep{debortoli2024schr}. 
\begin{wraptable}[11]{r}{0.25\textwidth}
\vspace{-4mm}
\centering
\scriptsize
    \caption{ Test FID on unpaired image transfer task. \textbf{(A)} EMNIST $\to$ MNIST, \textbf{(B)} AFHQ-64 Wild $\to$ Cat. }
    \setlength{\tabcolsep}{3.0pt}
\begin{tabular}{c|c c}
        \toprule
        Method & \textbf{(A)} & \textbf{(B)}\\
        \midrule
        IDBM~\citep{peluchetti2023diffusion} & 8.2 & -\\
        DSBM$^{\dag}$~\citep{shi2024diffusion} & \textbf{6.0} & \textbf{25.4} \\
        \midrule
        \textbf{DBFS} (Ours) & 9.1 & 44.4  \\
        \bottomrule
    \end{tabular}
\begin{tablenotes}
\item $\dag$  result from~\citep{debortoli2024schr}.
\end{tablenotes}\label{tab:2d_generation}
\end{wraptable}
Table~\ref{tab:2d_generation} shows that our method performs comparably to the finite-dimensional method. Additionally, we provide generated samples at various unseen resolutions in Figure~\ref{figure:2d_generatior_afhq} to demonstrate the resolution-invariant property of our infinite-dimensional models. We note that our method may have slightly lower FID scores compared to finite-dimensional baselines, which may align with the observation in~\citep{zhuang2023diffusion} that resolution-agnostic methods tend to have lower FID scores compared to resolution-specific ones. This could be because resolution-specific methods can incorporate domain-specific design features in their score networks. Samples generated from the reverse direction can be found in Figure~\ref{figure:2d_generatior_afhq_}.

\subsection{Bayesian Learning}
We validate our Bayesian learning algorithm for modeling functional data. Specifically, we will consider the temporal data as a function. We denote $\mb{Y}[\mb{O}] = \{\mb{Y}[\mb{p}_i]\}_{i=1}^{|\mb{O}|}$ as a collection of a function evaluation on a set of 1-dimensional observation grid $\mb{O} = \{\mb{p}_i\}_i^{|\mb{O}|}$ where $0 \leq \mb{p}_0 < \cdots < \mb{p}_{|\mb{O}|} \leq I$. We assume that each observed time series approximates a corresponding underlying continuous function $\mb{X} : \mathbb{R} \to \mathbb{R}^d$ as the number of observations increases $\ie \{\mb{Y}[\mb{p}_i]\}_{i=1}^{|\mb{O}| \to \infty}
 \approx \mathbf{X}$. For given observations $\mb{Y}[\mb{O}]$, our goal is to infer the posterior distribution on some set of unobserved grid $\mb{T} = [0, I] - \mb{O}$ $\ie \mathbb{P}(\mb{Y}[\mb{T}]|\mb{Y}[\mb{O}])$ and therefore modeling distribution over $\mb{X}$ on $[0, I].$ Please refer to Section~\ref{sec:modeling termpoal function_appx} for further details.

\paragraph{Functional Regression} 
To verify the effectiveness of the proposed DBFS in generating functions in the 1D domain, we conducted regression experiments using synthetic data generated from the Gaussian Process (GP) by following the experimental settings in \citep{lee2020bootstrapping}. Figure~\ref{regression_result} shows the sampled trajectories of a controlled dynamics $\mb{X}^{\alpha}_t$ for $t \in [0, \frac{T}{2}, T]$ trained on data generated from GP with RBF covariance kernel. The stochastic process begins from the deterministic function $\mb{X}^{\alpha}_0 = \mb{x}_0$ at $t=0$ and propagates towards the conditional posterior distribution $\mb{X}^{\alpha}_T \sim \mathbb{P}(\mb{Y}[\mb{T}]|\mb{Y}[\mb{O}])$ at $t=T$.

\begin{wraptable}[8]{r}{0.5\textwidth}
% \vspace{-5mm}
\centering
\scriptsize
\captionof{table}{Test imputation RMSE on Physionet.}
    \begin{tabular}{c|c|c|c}
        \toprule
        Model & $10\%$ & $50\%$ & $90\%$  \\
        \midrule
        CSDI$^{\dag}$~\citep{CSDI} & 0.60 $\pm$ 0.27 & 0.66 $\pm$ 0.06& 0.84 $\pm$ 0.04 \\
        DSDP-GP$^{\dag}$~\citep{bilovs2023modeling} & 0.52 $\pm$ 0.04 & 0.64 $\pm$ 0.05 & 0.81 $\pm$ 0.03 \\
        \midrule
        \textbf{DBFS} (Ours) & \textbf{0.50 $\pm$ 0.04 }& \textbf{0.61 $\pm$ 0.04 }& \textbf{0.77 $\pm$ 0.03 }\\
        \bottomrule
    \end{tabular}
\begin{tablenotes}
\item $\dag$ results from~\citep{bilovs2023modeling}.
\end{tablenotes}\label{table:physionet}
\vspace{-6mm}
\end{wraptable}
\paragraph{Imputation.} We evaluate our method against recent diffusion-based imputation method where the goal is to infer the conditional distribution $p(\mb{Y}[\mb{T}] | \mb{Y}[\mb{O}])$ of unobserved grid $\mb{Y}[\mb{T}]$ give observations $\mb{Y}[\mb{O}]$. CSDI~\citep{CSDI} utilizes DDPM~\citep{ho2020denoising} to learn the reverse process by treating the temporal data $\mb{Y}[\mb{O}]$ as a $\mathbb{R}^{|\mb{O}| \times d}$ dimensional feature. Extending this, DSDP-GP~\citep{bilovs2023modeling} enhances CSDI by incorporating noise derived from a stochastic process, instead of simple Gaussian noise. We maintained the same training setup as these models, including random seeds and the model architecture for control $\alpha^{\theta}$ in~\eqref{eq:controlled SDE}. Consistent with their methodology, we employed the Physionet dataset~\citep{Physionet}, which comprises medical time-series data collected on an hourly rate. Since the dataset inherently contains missing values, we selected certain degrees of observed values to create an imputation test set for evaluation. We then reported the results on this test set, varying the degrees of missingness. Table~\ref{table:physionet} shows that we outperform the previous methods even though it solely relies on forward propagation of controlled SDEs in~\eqref{eq:conditioned SDE} without denoising procedure.

\begin{table}[t]
\centering
\footnotesize
    \caption{Regression results. $``$context$"$ and $``$target$"$ refer to the log-likelhoods at $\mb{O}$ and $\mb{T}$, respectively.}
    \vspace{2mm}
\begin{tabular}{c|c c|c c|c c}
        \toprule
        \multirow{2}{*}{Method} & \multicolumn{2}{c}{RBF} & \multicolumn{2}{c}{Matérn 5/2} & \multicolumn{2}{c}{Periodic} \\
        \cmidrule{2-3} \cmidrule{4-5} \cmidrule{6-7}
        & context & target & context & target & context & target  \\
        \midrule
        CNP & 0.97 $\pm$ 0.01 & 0.45 $\pm$ 0.01& 0.85$\pm$ 0.01 & 0.21 $\pm$ 0.02 & -0.16$\pm$ 0.01 & -1.75 $\pm$ 0.02 \\
        NP & 0.90 $\pm$ 0.01 & 0.42 $\pm$ 0.01 & 0.77 $\pm$ 0.01 & 0.20 $\pm$ 0.03 & -0.18 $\pm$ 0.01 & -1.34 $\pm$ 0.03  \\
        \midrule
        \textbf{DBFS} & \textbf{1.02 $\pm$ 0.01 }& \textbf{0.47 $\pm$ 0.01 }& \textbf{0.93 $\pm$ 0.01 } & \textbf{0.25 $\pm$ 0.01 }  & \textbf{-0.15 $\pm$ 0.01 }  & -1.88 $\pm$ 0.02 \\
        \bottomrule
    \end{tabular}
\end{table}

\begin{figure}[t]
\begin{minipage}[b]{0.99\linewidth}
\begin{minipage}[b]{0.49\linewidth}
\includegraphics[width=1.\linewidth]{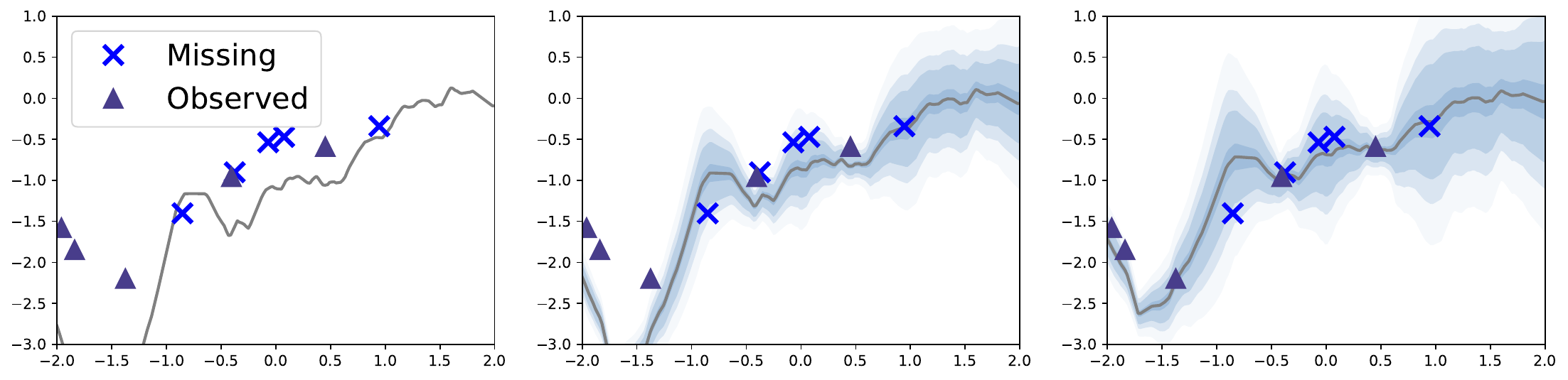}
\end{minipage}
\begin{minipage}[b]{0.49\linewidth}
\includegraphics[width=1.\linewidth]{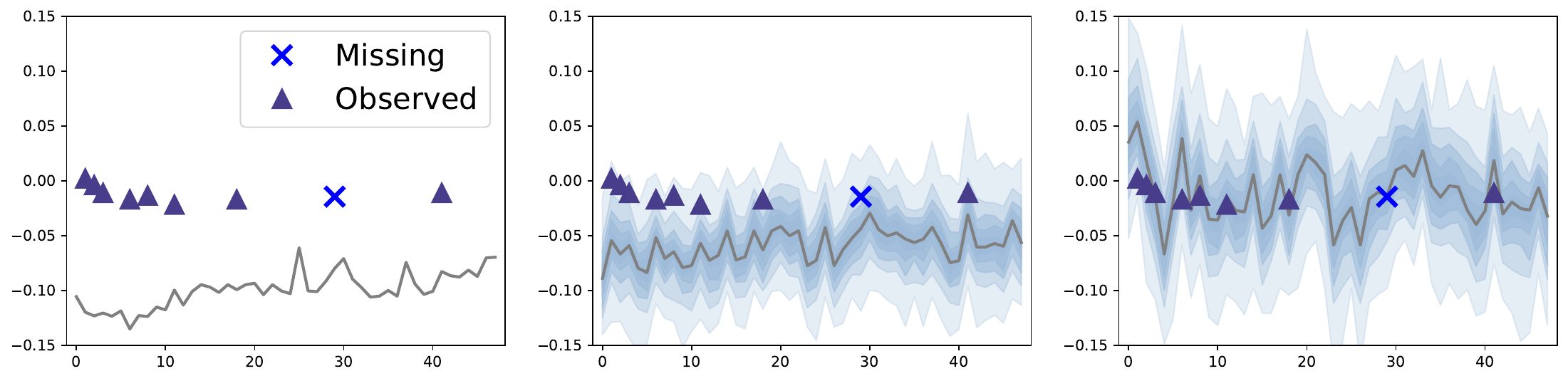}
\end{minipage}
\captionof{figure}{Sampled functions from a learned stochastic process $\mb{X}_t^{\alpha}$ evaluated on $[0, I]$ for $t \in [0, \frac{T}{2}, T]$. The grey line represents the mean function $\mathbb{E}[\mb{X}^{\alpha}_t]$ and the blue-shaded region represents the confidence interval. (Left) GP with RBF kernel. (Right) Physionet.}
\label{regression_result}
\vspace{-2mm}
\end{minipage}
\end{figure}

\section{Conclusion and Limitation}
In this work, we shed light on the application of the infinite-dimensional Doob's $h$-transform, exploiting SOC theory in infinite-dimensional spaces. By developing an explicit Radon-Nikodym density, we address the challenge posed by the absence of an equivalent to the Lebesgue measure. With specified cost functions for control objectives, it enables us to extend previous algorithm based on the finite-dimensional Doob's $h$-transform into infinite-dimensional function spaces, such as resolution-free unpaired image transfer and functional Bayesian posterior sampling.

Compared to the recent infinite-dimensional score-based diffusion model~\citep{lim2023scorebased}, our work restricts the coefficients for the stochastic dynamics to be time-independent. This limitation prevents us from defining a noise schedule for the diffusion model~\citep{zhou2024denoising}, which may hinder performance improvements. Additionally, in Bayesian learning, computing the gradient of the proposed training loss function \eqref{eq:relative entropy loss} can be computationally demanding. Thus, developing a more scalable algorithm would be an interesting direction for future work. Furthermore, as our model can be applied to any functional domain, we have limited our experiments to regular 1D and 2D domains, leaving the extension to more general domains for future work.

\section*{Broader Societal Impact}
Similar to other works in the literature, our proposed method holds the potential for both beneficial outcomes, such as automated data synthesis, and adverse implications, such as the deep fakes, depending on how it is used. We adhere to ethical standards for using our model in generative AI.

\section*{Acknowledgements}
This work was partly supported by Institute of Information \& communications Technology Planning \& Evaluation(IITP) grant funded by the Korea government(MSIT) (No.RS-2019-II190075, Artificial Intelligence Graduate School Program(KAIST), No.2022-0-00184, Development and Study of AI Technologies to Inexpensively Conform to Evolving Policy on Ethics, No. 2022-0-00612, Geometric and Physical Commonsense Reasoning based Behavior Intelligence for Embodied AI) and the National Research Foundation of Korea(NRF) grant funded by the Korea government(MSIT) (NRF-2021M3E5D9025030, NRF-2022R1A5A708390812, RS-2024-00410082).

\bibliographystyle{plain}
\bibliography{main}

\begin{thebibliography}{10}

\bibitem{ANDERSON1982313}
Brian~D.O. Anderson.
\newblock Reverse-time diffusion equation models.
\newblock {\em Stochastic Processes and their Applications}, 12(3):313--326,
  1982.

\bibitem{baker2024conditioning}
Elizabeth~Louise Baker, Gefan Yang, Michael~L Severinsen, Christy~Anna Hipsley,
  and Stefan Sommer.
\newblock Conditioning non-linear and infinite-dimensional diffusion processes.
\newblock {\em arXiv preprint arXiv:2402.01434}, 2024.

\bibitem{baldassari2024conditional}
Lorenzo Baldassari, Ali Siahkoohi, Josselin Garnier, Knut Solna, and Maarten~V
  de~Hoop.
\newblock Conditional score-based diffusion models for bayesian inference in
  infinite dimensions.
\newblock {\em Advances in Neural Information Processing Systems}, 36, 2024.

\bibitem{belopolskaya2012stochastic}
Y.I. Belopolskaya and Y.L. Dalecky.
\newblock {\em Stochastic Equations and Differential Geometry}.
\newblock Mathematics and its Applications. Springer Netherlands, 2012.

\bibitem{berner2022optimal}
Julius Berner, Lorenz Richter, and Karen Ullrich.
\newblock An optimal control perspective on diffusion-based generative
  modeling.
\newblock {\em Transactions on Machine Learning Research}, 2024.

\bibitem{bilovs2023modeling}
Marin Bilo{\v{s}}, Kashif Rasul, Anderson Schneider, Yuriy Nevmyvaka, and
  Stephan G{\"u}nnemann.
\newblock Modeling temporal data as continuous functions with stochastic
  process diffusion.
\newblock In {\em International Conference on Machine Learning}, pages
  2452--2470. PMLR, 2023.

\bibitem{bogachev2011uniqueness}
Vladimir Bogachev, Giuseppe~Da Prato, and Michael R{\"o}ckner.
\newblock Uniqueness for solutions of fokker--planck equations on infinite
  dimensional spaces.
\newblock {\em Communications in Partial Differential Equations},
  36(6):925--939, 2011.

\bibitem{debortoli2024schr}
Valentin~De Bortoli, Iryna Korshunova, Andriy Mnih, and Arnaud Doucet.
\newblock Schrodinger bridge flow for unpaired data translation.
\newblock In {\em The Thirty-eighth Annual Conference on Neural Information
  Processing Systems}, 2024.

\bibitem{carmona2018probabilistic}
R.~Carmona and F.~Delarue.
\newblock {\em Probabilistic Theory of Mean Field Games with Applications I:
  Mean Field FBSDEs, Control, and Games}.
\newblock Probability Theory and Stochastic Modelling. Springer International
  Publishing, 2018.

\bibitem{chaudhari2018deep}
Pratik Chaudhari, Adam Oberman, Stanley Osher, Stefano Soatto, and Guillaume
  Carlier.
\newblock Deep relaxation: partial differential equations for optimizing deep
  neural networks.
\newblock {\em Research in the Mathematical Sciences}, 5:1--30, 2018.

\bibitem{chen2024generative}
Tianrong Chen, Jiatao Gu, Laurent Dinh, Evangelos Theodorou, Joshua~M.
  Susskind, and Shuangfei Zhai.
\newblock Generative modeling with phase stochastic bridge.
\newblock In {\em The Twelfth International Conference on Learning
  Representations}, 2024.

\bibitem{chen2022likelihood}
Tianrong Chen, Guan-Horng Liu, and Evangelos Theodorou.
\newblock Likelihood training of schr\"odinger bridge using forward-backward
  {SDE}s theory.
\newblock In {\em International Conference on Learning Representations}, 2022.

\bibitem{chen2020stochastic}
Yongxin Chen, Tryphon~T. Georgiou, and Michele Pavon.
\newblock Stochastic control liaisons: Richard sinkhorn meets gaspard monge on
  a schrödinger bridge.
\newblock {\em SIAM Review}, 63(2):249--313, 2021.

\bibitem{chetrite2015nonequilibrium}
Rapha{\"e}l Chetrite and Hugo Touchette.
\newblock Nonequilibrium markov processes conditioned on large deviations.
\newblock In {\em Annales Henri Poincar{\'e}}, volume~16, pages 2005--2057.
  Springer, 2015.

\bibitem{choi2020starganv2}
Yunjey Choi, Youngjung Uh, Jaejun Yoo, and Jung-Woo Ha.
\newblock Stargan v2: Diverse image synthesis for multiple domains.
\newblock In {\em Proceedings of the IEEE Conference on Computer Vision and
  Pattern Recognition}, 2020.

\bibitem{DaPrato1997}
G.~Da~Prato and J.~Zabczyk.
\newblock Differentiability of the feynman-kac semigroup and a control
  application.
\newblock {\em Atti della Accademia Nazionale dei Lincei. Classe di Scienze
  Fisiche, Matematiche e Naturali. Rendiconti Lincei. Matematica e
  Applicazioni}, 8(3):183--188, 10 1997.

\bibitem{da2002second}
G.~Da~Prato and J.~Zabczyk.
\newblock {\em Second Order Partial Differential Equations in Hilbert Spaces}.
\newblock London Mathematical Society Lecture Note Series. Cambridge University
  Press, 2002.

\bibitem{da2014stochastic}
G.~Da~Prato and J.~Zabczyk.
\newblock {\em Stochastic Equations in Infinite Dimensions}.
\newblock Encyclopedia of Mathematics and its Applications. Cambridge
  University Press, 2014.

\bibitem{de2021diffusion}
Valentin De~Bortoli, James Thornton, Jeremy Heng, and Arnaud Doucet.
\newblock Diffusion schr{\"o}dinger bridge with applications to score-based
  generative modeling.
\newblock {\em Advances in Neural Information Processing Systems},
  34:17695--17709, 2021.

\bibitem{domingo2023stochastic}
Carles Domingo-Enrich, Jiequn Han, Brandon Amos, Joan Bruna, and Ricky~TQ Chen.
\newblock Stochastic optimal control matching.
\newblock {\em arXiv preprint arXiv:2312.02027}, 2023.

\bibitem{dupont2021coin}
Emilien Dupont, Adam Goli{\'n}ski, Milad Alizadeh, Yee~Whye Teh, and Arnaud
  Doucet.
\newblock Coin: Compression with implicit neural representations.
\newblock {\em arXiv preprint arXiv:2103.03123}, 2021.

\bibitem{dutordoir2023neural}
Vincent Dutordoir, Alan Saul, Zoubin Ghahramani, and Fergus Simpson.
\newblock Neural diffusion processes.
\newblock In {\em International Conference on Machine Learning}, pages
  8990--9012. PMLR, 2023.

\bibitem{fabbri2017stochastic}
Giorgio Fabbri, Fausto Gozzi, and Andrzej Swiech.
\newblock Stochastic optimal control in infinite dimension.
\newblock {\em Probability and Stochastic Modelling. Springer}, 2017.

\bibitem{fleming2006controlled}
Wendell~H Fleming and Halil~Mete Soner.
\newblock {\em Controlled Markov processes and viscosity solutions}, volume~25.
\newblock Springer Science \& Business Media, 2006.

\bibitem{franzese2024continuous}
Giulio Franzese, Giulio Corallo, Simone Rossi, Markus Heinonen, Maurizio
  Filippone, and Pietro Michiardi.
\newblock Continuous-time functional diffusion processes.
\newblock {\em Advances in Neural Information Processing Systems}, 36, 2024.

\bibitem{fuhrman2003class}
Marco Fuhrman.
\newblock A class of stochastic optimal control problems in hilbert spaces:
  Bsdes and optimal control laws, state constraints, conditioned processes.
\newblock {\em Stochastic processes and their applications}, 108(2):263--298,
  2003.

\bibitem{pmlr-v80-garnelo18a}
Marta Garnelo, Dan Rosenbaum, Christopher Maddison, Tiago Ramalho, David
  Saxton, Murray Shanahan, Yee~Whye Teh, Danilo Rezende, and S.~M.~Ali Eslami.
\newblock Conditional neural processes.
\newblock In Jennifer Dy and Andreas Krause, editors, {\em Proceedings of the
  35th International Conference on Machine Learning}, volume~80 of {\em
  Proceedings of Machine Learning Research}, pages 1704--1713. PMLR, 10--15 Jul
  2018.

\bibitem{garnelo2018neural}
Marta Garnelo, Jonathan Schwarz, Dan Rosenbaum, Fabio Viola, Danilo~J. Rezende,
  S.~M.~Ali Eslami, and Yee~Whye Teh.
\newblock Neural processes, 2018.

\bibitem{gawarecki2010stochastic}
L.~Gawarecki and V.~Mandrekar.
\newblock {\em Stochastic Differential Equations in Infinite Dimensions: with
  Applications to Stochastic Partial Differential Equations}.
\newblock Probability and Its Applications. Springer Berlin Heidelberg, 2010.

\bibitem{Physionet}
A.~L. Goldberger, L.~A.~N. Amaral, L.~Glass, J.~M. Hausdorff, P.~Ch. Ivanov,
  R.~G. Mark, J.~E. Mietus, G.~B. Moody, C.-K. Peng, and H.~E. Stanley.
\newblock {PhysioBank, PhysioToolkit, and PhysioNet}: Components of a new
  research resource for complex physiologic signals.
\newblock {\em Circulation}, 101(23):e215--e220, 2000 (June 13).
\newblock Circulation Electronic Pages:
  http://circ.ahajournals.org/content/101/23/e215.full PMID:1085218; doi:
  10.1161/01.CIR.101.23.e215.

\bibitem{goldys2008ornstein}
Ben Goldys and Bohdan Maslowski.
\newblock The ornstein--uhlenbeck bridge and applications to markov semigroups.
\newblock {\em Stochastic processes and their applications},
  118(10):1738--1767, 2008.

\bibitem{hagemann2023multilevel}
Paul Hagemann, Sophie Mildenberger, Lars Ruthotto, Gabriele Steidl, and
  Nicole~Tianjiao Yang.
\newblock Multilevel diffusion: Infinite dimensional score-based diffusion
  models for image generation.
\newblock {\em arXiv preprint arXiv:2303.04772}, 2023.

\bibitem{hartmann2019variational}
Carsten Hartmann, Omar Kebiri, Lara Neureither, and Lorenz Richter.
\newblock Variational approach to rare event simulation using least-squares
  regression.
\newblock {\em Chaos: An Interdisciplinary Journal of Nonlinear Science},
  29(6), 2019.

\bibitem{ho2020denoising}
Jonathan Ho, Ajay Jain, and Pieter Abbeel.
\newblock Denoising diffusion probabilistic models.
\newblock {\em Advances in neural information processing systems},
  33:6840--6851, 2020.

\bibitem{holdijk2022path}
Lars Holdijk, Yuanqi Du, Priyank Jaini, Ferry Hooft, Bernd Ensing, and Max
  Welling.
\newblock Path integral stochastic optimal control for sampling transition
  paths.
\newblock In {\em ICML 2022 2nd AI for Science Workshop}, 2022.

\bibitem{jaegle2022perceiver}
Andrew Jaegle, Sebastian Borgeaud, Jean-Baptiste Alayrac, Carl Doersch, Catalin
  Ionescu, David Ding, Skanda Koppula, Daniel Zoran, Andrew Brock, Evan
  Shelhamer, Olivier~J Henaff, Matthew Botvinick, Andrew Zisserman, Oriol
  Vinyals, and Joao Carreira.
\newblock Perceiver {IO}: A general architecture for structured inputs \&
  outputs.
\newblock In {\em International Conference on Learning Representations}, 2022.

\bibitem{kovachki2023neural}
Nikola Kovachki, Zongyi Li, Burigede Liu, Kamyar Azizzadenesheli, Kaushik
  Bhattacharya, Andrew Stuart, and Anima Anandkumar.
\newblock Neural operator: Learning maps between function spaces with
  applications to pdes.
\newblock {\em Journal of Machine Learning Research}, 24(89):1--97, 2023.

\bibitem{lee2020bootstrapping}
Juho Lee, Yoonho Lee, Jungtaek Kim, Eunho Yang, Sung~Ju Hwang, and Yee~Whye
  Teh.
\newblock Bootstrapping neural processes.
\newblock {\em Advances in neural information processing systems},
  33:6606--6615, 2020.

\bibitem{li2022transformer}
Zijie Li, Kazem Meidani, and Amir~Barati Farimani.
\newblock Transformer for partial differential equations{\textquoteright}
  operator learning.
\newblock {\em Transactions on Machine Learning Research}, 2023.

\bibitem{li2020fourier}
Zongyi Li, Nikola~Borislavov Kovachki, Kamyar Azizzadenesheli, Burigede liu,
  Kaushik Bhattacharya, Andrew Stuart, and Anima Anandkumar.
\newblock Fourier neural operator for parametric partial differential
  equations.
\newblock In {\em International Conference on Learning Representations}, 2021.

\bibitem{lim2023score}
Jae~Hyun Lim, Nikola~B Kovachki, Ricardo Baptista, Christopher Beckham, Kamyar
  Azizzadenesheli, Jean Kossaifi, Vikram Voleti, Jiaming Song, Karsten Kreis,
  Jan Kautz, et~al.
\newblock Score-based diffusion models in function space.
\newblock {\em arXiv preprint arXiv:2302.07400}, 2023.

\bibitem{lim2023scorebased}
Sungbin Lim, Eunbi Yoon, Taehyun Byun, Taewon Kang, Seungwoo Kim, Kyungjae Lee,
  and Sungjoon Choi.
\newblock Score-based generative modeling through stochastic evolution
  equations in hilbert spaces.
\newblock In {\em Thirty-seventh Conference on Neural Information Processing
  Systems}, 2023.

\bibitem{liu2022deep}
Guan-Horng Liu, Tianrong Chen, Oswin So, and Evangelos Theodorou.
\newblock Deep generalized schr\"odinger bridge.
\newblock In Alice~H. Oh, Alekh Agarwal, Danielle Belgrave, and Kyunghyun Cho,
  editors, {\em Advances in Neural Information Processing Systems}, 2022.

\bibitem{liu2024generalized}
Guan-Horng Liu, Yaron Lipman, Maximilian Nickel, Brian Karrer, Evangelos
  Theodorou, and Ricky T.~Q. Chen.
\newblock Generalized schr\"odinger bridge matching.
\newblock In {\em The Twelfth International Conference on Learning
  Representations}, 2024.

\bibitem{liu2023i2sb}
Guan-Horng Liu, Arash Vahdat, De-An Huang, Evangelos~A Theodorou, Weili Nie,
  and Anima Anandkumar.
\newblock I{$^2$}sb: Image-to-image schr{\"o}dinger bridge.
\newblock {\em arXiv preprint arXiv:2302.05872}, 2023.

\bibitem{liu2023learning}
Xingchao Liu, Lemeng Wu, Mao Ye, and qiang liu.
\newblock Learning diffusion bridges on constrained domains.
\newblock In {\em The Eleventh International Conference on Learning
  Representations}, 2023.

\bibitem{506400}
S.K. Mitter.
\newblock Filtering and stochastic control: a historical perspective.
\newblock {\em IEEE Control Systems Magazine}, 16(3):67--76, 1996.

\bibitem{nusken2021solving}
Nikolas N{\"u}sken and Lorenz Richter.
\newblock Solving high-dimensional hamilton--jacobi--bellman pdes using neural
  networks: perspectives from the theory of controlled diffusions and measures
  on path space.
\newblock {\em Partial differential equations and applications}, 2:1--48, 2021.

\bibitem{peebles2023scalable}
William Peebles and Saining Xie.
\newblock Scalable diffusion models with transformers.
\newblock In {\em Proceedings of the IEEE/CVF International Conference on
  Computer Vision}, pages 4195--4205, 2023.

\bibitem{peluchetti2022nondenoising}
Stefano Peluchetti.
\newblock Non-denoising forward-time diffusions, 2022.

\bibitem{peluchetti2023diffusion}
Stefano Peluchetti.
\newblock Diffusion bridge mixture transports, schr{\"o}dinger bridge problems
  and generative modeling.
\newblock {\em Journal of Machine Learning Research}, 24(374):1--51, 2023.

\bibitem{phillips2022spectral}
Angus Phillips, Thomas Seror, Michael Hutchinson, Valentin De~Bortoli, Arnaud
  Doucet, and Emile Mathieu.
\newblock Spectral diffusion processes.
\newblock {\em arXiv preprint arXiv:2209.14125}, 2022.

\bibitem{pidstrigach2023infinite}
Jakiw Pidstrigach, Youssef Marzouk, Sebastian Reich, and Sven Wang.
\newblock Infinite-dimensional diffusion models for function spaces.
\newblock {\em arXiv e-prints}, pages arXiv--2302, 2023.

\bibitem{Pra1991ASC}
Paolo~Dai Pra.
\newblock A stochastic control approach to reciprocal diffusion processes.
\newblock {\em Applied Mathematics and Optimization}, 23:313--329, 1991.

\bibitem{reich2019data}
Sebastian Reich.
\newblock Data assimilation: the schr{\"o}dinger perspective.
\newblock {\em Acta Numerica}, 28:635--711, 2019.

\bibitem{richter2021solving}
Lorenz Richter.
\newblock {\em Solving high-dimensional PDEs, approximation of path space
  measures and importance sampling of diffusions}.
\newblock PhD thesis, BTU Cottbus-Senftenberg, 2021.

\bibitem{richter2024improved}
Lorenz Richter and Julius Berner.
\newblock Improved sampling via learned diffusions.
\newblock In {\em The Twelfth International Conference on Learning
  Representations}, 2024.

\bibitem{rissanen2023generative}
Severi Rissanen, Markus Heinonen, and Arno Solin.
\newblock Generative modelling with inverse heat dissipation.
\newblock In {\em The Eleventh International Conference on Learning
  Representations}, 2023.

\bibitem{rogers2000diffusions}
L~Chris~G Rogers and David Williams.
\newblock {\em Diffusions, Markov processes and martingales: Volume 2, It{\^o}
  calculus}, volume~2.
\newblock Cambridge university press, 2000.

\bibitem{särkkä2019applied}
S.~S{\"a}rkk{\"a} and A.~Solin.
\newblock {\em Applied Stochastic Differential Equations}.
\newblock Institute of Mathematical Statistics Textbooks. Cambridge University
  Press, 2019.

\bibitem{shi2024diffusion}
Yuyang Shi, Valentin De~Bortoli, Andrew Campbell, and Arnaud Doucet.
\newblock Diffusion schr{\"o}dinger bridge matching.
\newblock {\em Advances in Neural Information Processing Systems}, 36, 2024.

\bibitem{doi:10.1080/07362999308809319}
Isabel Simão.
\newblock Regular transition densities for infinite dimensional diffusions.
\newblock {\em Stochastic Analysis and Applications}, 11(3):309--336, 1993.

\bibitem{sitzmann2020implicit}
Vincent Sitzmann, Julien Martel, Alexander Bergman, David Lindell, and Gordon
  Wetzstein.
\newblock Implicit neural representations with periodic activation functions.
\newblock {\em Advances in neural information processing systems},
  33:7462--7473, 2020.

\bibitem{song2021scorebased}
Yang Song, Jascha Sohl-Dickstein, Diederik~P Kingma, Abhishek Kumar, Stefano
  Ermon, and Ben Poole.
\newblock Score-based generative modeling through stochastic differential
  equations.
\newblock In {\em International Conference on Learning Representations}, 2021.

\bibitem{sun2019functional}
Shengyang Sun, Guodong Zhang, Jiaxin Shi, and Roger Grosse.
\newblock {FUNCTIONAL} {VARIATIONAL} {BAYESIAN} {NEURAL} {NETWORKS}.
\newblock In {\em International Conference on Learning Representations}, 2019.

\bibitem{suzuki2020generalization}
Taiji Suzuki.
\newblock Generalization bound of globally optimal non-convex neural network
  training: Transportation map estimation by infinite dimensional langevin
  dynamics.
\newblock {\em Advances in Neural Information Processing Systems},
  33:19224--19237, 2020.

\bibitem{NEURIPS2020_55053683}
Matthew Tancik, Pratul Srinivasan, Ben Mildenhall, Sara Fridovich-Keil, Nithin
  Raghavan, Utkarsh Singhal, Ravi Ramamoorthi, Jonathan Barron, and Ren Ng.
\newblock Fourier features let networks learn high frequency functions in low
  dimensional domains.
\newblock In H.~Larochelle, M.~Ranzato, R.~Hadsell, M.F. Balcan, and H.~Lin,
  editors, {\em Advances in Neural Information Processing Systems}, volume~33,
  pages 7537--7547. Curran Associates, Inc., 2020.

\bibitem{CSDI}
Yusuke Tashiro, Jiaming Song, Yang Song, and Stefano Ermon.
\newblock Csdi: Conditional score-based diffusion models for probabilistic time
  series imputation.
\newblock In {\em Advances in Neural Information Processing Systems}, 2021.

\bibitem{8618948}
Evangelos~A. Theodorou, George~I. Boutselis, and Kaivalya Bakshi.
\newblock Linearly solvable stochastic optimal control for infinite-dimensional
  systems.
\newblock In {\em 2018 IEEE Conference on Decision and Control (CDC)}, pages
  4110--4116, 2018.

\bibitem{tong2024improving}
Alexander Tong, Kilian FATRAS, Nikolay Malkin, Guillaume Huguet, Yanlei Zhang,
  Jarrid Rector-Brooks, Guy Wolf, and Yoshua Bengio.
\newblock Improving and generalizing flow-based generative models with
  minibatch optimal transport.
\newblock {\em Transactions on Machine Learning Research}, 2024.
\newblock Expert Certification.

\bibitem{tzen2019theoretical}
Belinda Tzen and Maxim Raginsky.
\newblock Theoretical guarantees for sampling and inference in generative
  models with latent diffusions.
\newblock In {\em COLT}, 2019.

\bibitem{van2007stochastic}
Ramon Van~Handel.
\newblock Stochastic calculus, filtering, and stochastic control.
\newblock {\em Course notes., URL http://www. princeton.
  edu/rvan/acm217/ACM217. pdf}, 14, 2007.

\bibitem{vargas2023bayesian}
Francisco Vargas, Andrius Ovsianas, David Fernandes, Mark Girolami, Neil~D
  Lawrence, and Nikolas N{\"u}sken.
\newblock Bayesian learning via neural schr{\"o}dinger--f{\"o}llmer flows.
\newblock {\em Statistics and Computing}, 33(1):3, 2023.

\bibitem{wang2019function}
Ziyu Wang, Tongzheng Ren, Jun Zhu, and Bo~Zhang.
\newblock Function space particle optimization for {B}ayesian neural networks.
\newblock In {\em International Conference on Learning Representations}, 2019.

\bibitem{xie2022neural}
Yiheng Xie, Towaki Takikawa, Shunsuke Saito, Or~Litany, Shiqin Yan, Numair
  Khan, Federico Tombari, James Tompkin, Vincent Sitzmann, and Srinath Sridhar.
\newblock Neural fields in visual computing and beyond.
\newblock {\em Computer Graphics Forum}, 41(2):641--676, 2022.

\bibitem{ye2022first}
Mao Ye, Lemeng Wu, and qiang liu.
\newblock First hitting diffusion models for generating manifold, graph and
  categorical data.
\newblock In Alice~H. Oh, Alekh Agarwal, Danielle Belgrave, and Kyunghyun Cho,
  editors, {\em Advances in Neural Information Processing Systems}, 2022.

\bibitem{zhang2022path}
Qinsheng Zhang and Yongxin Chen.
\newblock Path integral sampler: A stochastic control approach for sampling.
\newblock In {\em International Conference on Learning Representations}, 2022.

\bibitem{zhou2024denoising}
Linqi Zhou, Aaron Lou, Samar Khanna, and Stefano Ermon.
\newblock Denoising diffusion bridge models.
\newblock In {\em The Twelfth International Conference on Learning
  Representations}, 2024.

\bibitem{zhuang2023diffusion}
Peiye Zhuang, Samira Abnar, Jiatao Gu, Alex Schwing, Joshua~M. Susskind, and
  Miguel~{\'A}ngel Bautista.
\newblock Diffusion probabilistic fields.
\newblock In {\em The Eleventh International Conference on Learning
  Representations}, 2023.

\end{thebibliography}

\newpage
%%%%%%%%%%%%%%%%%%%%%%%%%%%%%%%%%%%%%%%%%%%%%%%%%%%%%%%%%%%%
\appendix
\newpage
\onecolumn

\setcounter{figure}{0}
\setcounter{table}{0}
\setcounter{equation}{0}
\setcounter{theorem}{0}

\counterwithin{figure}{section}
\counterwithin{table}{section}
\counterwithin{equation}{section}
\counterwithin{theorem}{section}
% \counterwithin{definition}{section}
% \counterwithin{lemma}{section}

\section{Appendix}

\subsection{Verification Theorem and Markov control}
For the derivation, we will use the following assumption
\begin{assumption}\label{assumptions value function}
The function $\mc{V}: [0, T]\times \mc{H} \to \mathbb{R}$ and its derivatives $D_{\mb{x}} \mc{V}, D_{\mb{xx}} \mc{V}, \partial_t \mc{V}$ are uniformly continuous on bounded subsets of $[0, T] \times \mc{H}$ and $(0, T) \times \mc{H}$, respectively.
    Moreover, for all $(t, \mb{x}) \in (0, T) \times \mc{H}
    $, there exists $C_1, C_2 >0$ such that
    \begin{equation}
        |\mc{V}(t, \mb{x})| + |D_{\mb{x}}\mc{V}(t, \mb{x})| + |\partial_t \mc{V}(t, \mb{x})| + \norm{D_{\mb{xx}} \mc{V}(t, \mb{x})} + |\mc{A}^{\star} D_{\mb{x}} \mc{V}(t, \mb{x})| \leq C_1(1+ |\mb{x}|)^{C_2},
        \end{equation}
        where $\mc{A}^{\star}$ is adjoint operator of $\mc{A}$.
\end{assumption}
The HJB equation~\eqref{eq:PDE HJB} can be derived by the following theorem.
\begin{theorem}\label{DPP APPX} Let assumptions~\ref{assumptions value function} hold and let the function $\mc{V}$ with $\mc{V}(T, \mb{x}) = G(\mb{x})$ satisfying the dynamic programming principle for every $0 < t < t' < T$, $x \in \mc{H}$
\begin{equation}
    \mc{V}(t, \mb{x}) = \mathbb{E}_{\mathbb{P}^{\alpha}}\left[\int_t^{t'} \left[l(s, \mb{X}^{\alpha}_{t}) + \psi(\alpha_s)\right]ds + \mc{V}(t', \mb{X}_{t'})  | \mb{X}^{\alpha}_t = \mb{x}\right].
\end{equation}
Then $\mc{V}$ is a solution of the following equation:
\begin{equation}\label{eq:HJB APPX}
    \partial_t \mc{V} + \mc{L}\mc{V} + \inf_{\alpha \in \mc{U}} \left[ \la \sigma Q^{1/2} _{\mb{x}}\mc{V},  \alpha \ra + \frac{1}{2}\norm{\alpha}^2 \right] = 0, \quad \mc{V}(T, \mb{x}) = G(\mb{x}),
\end{equation}
\end{theorem}
\begin{proof}
The proof can be found in~\citep[Theorem~2.34]{fabbri2017stochastic}. 
\end{proof}

\subsubsection{Proof of Lemma~\ref{lemma:verification}}
\begin{proof}
To begin the proof, we can formally compute the minimum of $F(D_{\mb{x}}\mc{V})$. Since~\citep{DaPrato1997}
\begin{equation}
F(\mb{x}) = \inf_{\alpha \in \mc{U}} \left[ \la \mb{x}, \alpha \ra + \frac{1}{2}\norm{\alpha}^2\right] = -\frac{1}{2}\norm{\mb{x}}^2.
\end{equation}
Therefore, $F(\sigma Q^{1/2} D_{\mb{x}}\mc{V})$ takes the infimum at $\alpha^{*} = -\sigma Q^{1/2}D_{\mb{x}}\mc{V}$. Next, applying Itô's formula~\citep[Proposition~1.165]{fabbri2017stochastic} to $\mc{V}$ and taking expectation on both sides, we get
\begin{align}\label{eq:V function Ito} \mathbb{E}_{\mathbb{P}^{\alpha}}^{t, \mb{x}}\left[\mc{V}(T, \mb{X}^{\alpha}_T)\right] = \mc{V}(t, \mb{x}) + \mathbb{E}_{\mathbb{P}^{\alpha}}^{t, \mb{x}}\left[\int_t^T \left(\partial_t \mc{V}(s, \mb{X}^{\alpha}_s) + \mc{L}\mc{V}(s, \mb{X}_s^{\alpha}) + \la \sigma Q^{1/2} D_{\mb{x}}\mc{V}(s, \mb{X}^{\alpha}_s),  \alpha_s \ra \right)ds\right],
\end{align}
where we denote $\mathbb{E}_{\mathbb{P}^{\alpha}}^{t, \mb{x}}\left[\cdot\right] = \mathbb{E}_{\mathbb{P}^{\alpha}}\left[\cdot | \mb{X}_t=\mb{x}\right]$
By incorporating the fact that $\mathcal{V}$ satisfies the equation in~(\ref{eq:HJB APPX}), we can derive the following by adding $\mathbb{E}_{\mathbb{P}^{\alpha}}^{t, \mb{x}}\left[\int_t^T \frac{1}{2}\norm{\alpha_s}^2 ds\right]$ to both terms. The LHS of the equation~\eqref{eq:V function Ito} becomes:
\begin{equation}    \mathbb{E}_{\mathbb{P}^{\alpha}}^{t, \mb{x}}\left[\underbrace{\mc{V}(T, \mb{X}^{\alpha}_T)}_{= G(\mb{X}^{\alpha}_T)}\right]  + \mathbb{E}_{\mathbb{P}^{\alpha}}^{t, \mb{x}}\left[\int_t^T \frac{1}{2}\norm{\alpha_s}^2 ds\right] = \mathcal{J}(t, \mb{x}, \alpha).
\end{equation}
And for the RHS of the equation~\eqref{eq:V function Ito}:
\begin{align}
    &\mc{V}(t, \mb{x}) + \mathbb{E}_{\mathbb{P}^{\alpha}}^{t, \mb{x}}\left[\int_t^T \frac{1}{2}\norm{\alpha_s}^2 ds\right] + \mathbb{E}_{\mathbb{P}^{\alpha}}^{t, \mb{x}}\bigg[\int_t^T \left(\partial_t \mc{V}(s, \mb{X}^{\alpha}_s) + \mathcal{L}\mc{V}(s, \mb{X}_s^{\alpha}) + \la \sigma Q^{1/2} D_{\mb{x}}\mc{V}(s, \mb{X}^{\alpha}_s),  \alpha_s \ra\right) ds \bigg]  \nonumber \\
    & = \mc{V}(t, \mb{x}) + \mathbb{E}_{\mathbb{P}^{\alpha}}^{t, \mb{x}}\left[\int_t^T \left(\partial_t \mc{V}(s, \mb{X}^{\alpha}_s) + \mc{L}\mc{V}(s, \mb{X}_s^{\alpha}) + \left[ \la \sigma Q^{1/2} D_{\mb{x}}\mc{V}(s, \mb{X}^{\alpha}_s), \alpha_s \ra + \frac{1}{2}\norm{\alpha_s}^2 \right] \right)ds\right]
    \nonumber \\
    & =  \mc{V}(t, \mb{x}) + \mathbb{E}_{\mathbb{P}^{\alpha}}^{t, \mb{x}}\left[\int_t^T \left(\left[ \la \sigma Q^{1/2} D_{\mb{x}}\mc{V}(s, \mb{X}^{\alpha}_s), \alpha_s \ra + \frac{1}{2}\norm{\alpha_s}^2\right] - F(\sigma Q^{1/2}D_{\mb{x}}\mc{V}(s, \mb{X}_s^{\alpha})) \right)ds\right],
\end{align}
where the last equation can be derived by adding and subtracting $\mathbb{E}_{\mathbb{P}^{\alpha}}^{t,\mb{x}} \left[ \int_t^T F(D_{\mb{x}}\mc{V}(s, \mb{X}_s^{\alpha}) ds \right]$ and incorporating the fact that $\mathcal{V}$ satisfies the equation in~(\ref{eq:HJB APPX}) again. Hence, we get the following equation:
\begin{equation}\label{eq: HJB derivation 1}
    \mathcal{J}(t, \mb{x}, \alpha) = \mc{V}(t, \mb{x}) + \mathbb{E}_{\mathbb{P}^{\alpha}}^{t, \mb{x}}\left[\int_t^T \left(\left[ \la  \sigma Q^{1/2}  D_{\mb{x}}\mc{V}(s, \mb{X}^{\alpha}_s), \alpha_s \ra + \frac{1}{2}\norm{\alpha_s}^2 \right] - F(\sigma Q^{1/2}D_{\mb{x}}\mc{V}(s, \mb{X}_s^{\alpha}))\right)ds\right]
\end{equation}
Since, by definition
\begin{equation}
    \left[ \la \sigma Q^{1/2} D_{\mb{x}}\mc{V}(s, \mb{X}^{\alpha}_s), \alpha_s \ra + \frac{1}{2}\norm{\alpha_s}^2 \right] 
 - F(D_{\mb{x}}\mc{V}(s, \mb{X}_s^{\alpha})) \geq  0
\end{equation}
Therefore, by taking the infimum over $\alpha \in \mc{U}$ in the RHS of~\eqref{eq: HJB derivation 1},
\begin{equation}
    \mathcal{J}(t, \mb{x}, \alpha) \geq \mc{V}(t, \mb{x})
\end{equation}
Moreover, since we already verified that the $F(\sigma Q^{1/2} D_{\mb{x}}\mc{V})$ has infimum at $\alpha^* = - \sigma Q^{1/2} D_{\mb{x}}\mc{V}$. Therefore, by choosing $u = \alpha^*$, we have
\begin{equation}
    \left[ \la  \sigma Q^{1/2} D_{\mb{x}}\mc{V}(s, \mb{X}^{u}_s), u_s \ra + \frac{1}{2}\norm{u_s}^2 \right] 
 - F( \sigma Q^{1/2} D_{\mb{x}}\mc{V}(s, \mb{X}_s^{u})) = 0
\end{equation}
Thus, we get
\begin{equation}
    \mathcal{J}(t, \mb{x}, u) = \mc{V}(t, \mb{x}).
\end{equation}
Therefore, together with~\eqref{eq: HJB derivation 1}, this implies that $(\alpha^{*}, \mb{X}^{\alpha^{*}})$ is optimal at $(t, \mb{x}) \in [0, T] \times \mc{H}$ This concludes the proof.
 \end{proof}

\subsubsection{Markov Control Formulation}\label{subsection:Markov Control}

Now, we introduce the following corollary that states the Markov control formulation.
\begin{corollary}[Markov Control~\citep{fabbri2017stochastic}]\label{corollary:markov control} Let us consider the measurable function $\phi_t : (t, T) \times \mc{H} \to \mc{U}$ which admit a mild solution $\mb{X}^{\phi_t}$ of the following closed-loop equation:
\begin{equation}
    d\mathbf{X}^{\phi_t}_{s} = \left[\mc{A}\mb{X}^{\phi_s}_s + \sigma Q^{1/2}\phi_s(s, \mb{X}^{\phi_s}_s)\right]ds + d\mb{W}_t^Q, \; \mb{X}_t = \mb{x}
\end{equation}
Then the pair $(\alpha^{\phi_t}, \mb{X}^{\phi_t})$, where the control $\alpha^{\phi_t}$ is defined by the Markov feedback law $\alpha^{\phi_t}_s = \phi(s, \mb{X}^{\phi_t}_s)$ is admissible and it is optimal at $(t, \mb{x})$ for all $s \in [t, T]$.
\end{corollary}
Therefore, in the context of the initial value problem, such as in our case, we consider the form of the Markov control $\alpha_s := \alpha^{\phi_0}_s = \phi(s, \mb{X}^{\phi_0}_s)$ for $s \in [0, T]$. The proof and details can be found in~\citep[Chap~2.5.1]{fabbri2017stochastic}.

\subsection{Proof of Theorem~\ref{Theorem:Hopf-Cole}}
\begin{proof}
Let us consider the function $\mc{V}(t, v) = -\log h(t, v)$. 

\begin{equation}
    \partial_t h = -h \partial_t \mc{V}, \quad D h = -h D_{\mb{x}}\mc{V}, \quad D^2 h = h D_{\mb{x}}\mc{V} \otimes D_{\mb{x}}\mc{V} - h D_{\mb{xx}} \mc{V},
\end{equation}
Recall that $h$ satisfy the KBE in equation~\eqref{eq:PDE KBE}:
\begin{equation}
    \partial_t h + \mathcal{L} h = 0, \quad h_T = \tilde{G}.
\end{equation}
Since $h = e^{-\mc{V}}$, hence $\partial_t h = -\mathcal{L}h = \partial_t e^{-\mc{V}} = -\partial_t \mc{V} h$. Then,
\begin{align}
    \partial_t \mc{V} h &= \mathcal{L}h  \\
    & = \la D_{\mb{x}}h, \mc{A} \mb{X}_t \ra + \frac{1}{2}\text{Tr}\left[\sigma^2 D_{\mb{xx}} h Q \right] \\
    & = -\la h D_{\mb{x}}\mc{V}, \mc{A} \mb{X}_t  \ra  +\frac{1}{2}\text{Tr}\left[\sigma^2 \left( h D_{\mb{x}}\mc{V} \otimes D_{\mb{x}} \mc{V} - h D_{\mb{xx}} \mc{V}\right)Q \right] \\
    & = -\la h D_{\mb{x}}\mc{V}, \mc{A} \mb{X}_t  \ra +\frac{1}{2}\text{Tr}\left[\sigma^2 \left( h D_{\mb{x}} \mc{V} \otimes D_{\mb{x}} \mc{V}\right)Q\right] - \frac{1}{2}\text{Tr}\left[\sigma^2 h D_{\mb{xx}} \mc{V} Q\right].
\end{align}
We can simplify the last equation as  
\begin{align}
    \partial_t \mc{V} &= -\la D_{\mb{x}}\mc{V}, \mc{A} \mb{X}_t  \ra + \frac{1}{2}\text{Tr}\left[\sigma^2\left(D_{\mb{x}} \mc{V} \otimes D_{\mb{x}} \mc{V}\right)Q\right] - \frac{1}{2}\text{Tr}\left[\sigma^2 D_{\mb{xx}} \mc{V} Q\right] \\
    & = -\mathcal{L} \mc{V} + \frac{1}{2}\text{Tr}\left[\sigma^2\left(D_{\mb{x}} \mc{V}\otimes D_{\mb{x}} \mc{V}\right)Q\right].
\end{align}
Following~\citep{8618948}, the second term of RHS can be derived as follows
\begin{align}
    \frac{1}{2}\text{Tr}\left[\sigma^2\left(D_{\mb{x}} \mc{V}\otimes D_{\mb{x}} \mc{V}\right)Q\right] &= \frac{1}{2}\sum_{k \in \mathbb{N}} \la \sigma^2 (D_{\mb{x}} \mc{V} \otimes D_{\mb{x}}\mc{V}) Q \phi^{(k)}, \phi^{(k)}\ra\\
    & = \frac{1}{2}\sum_{k \in \mathbb{N}} \la \sigma^2  D_{\mb{x}} \mc{V} \la D_{\mb{x}}\mc{V}, Q \phi^{(k)}\ra, \phi^{(k)}  \ra  \\
    & = \frac{1}{2}\sum_{k \in \mathbb{N}} \la \sigma^2  D_{\mb{x}}\mc{V}, Q \phi^{(k)}\ra \la D_{\mb{x}}\mc{V}, \phi^{(k)}  \ra  \\
    & = \frac{1}{2}\sum_{k \in \mathbb{N}} \la \sigma^2  Q D_{\mb{x}}\mc{V}, \phi^{(k)}\ra \la D_{\mb{x}}\mc{V}, \phi^{(k)}  \ra  \\
    & = \frac{1}{2} \la \sigma^2  D_{\mb{x}}\mc{V}, Q D_{\mb{x}}\mc{V} \ra \\
    & = \frac{1}{2} \norm{\sigma Q^{1/2} D_{\mb{x}}\mc{V}}_{\mc{H}}^2
\end{align}
Therefore, combining the above results, we have
\begin{equation}\label{eq:PDE HJB APPX}
\partial_t \mc{V}  + \mathcal{L} \mc{V} - \frac{1}{2} \norm{\sigma Q^{1/2 }D_{\mb{x}}\mc{V}}_{\mc{H}}^2, \quad \mc{V}_T = G.
\end{equation}
Since~\eqref{eq:PDE HJB APPX} coincides with~\eqref{eq:PDE HJB} with $\psi(\cdot) := \frac{1}{2}\norm{\cdot}^2_{\mc{H}}$, this concludes the proof.
\end{proof}

\subsection{Proof of Theorem~\ref{Theorem:infinite dimensional density}}
\begin{proof}
A proof of Theorem~\ref{Theorem:infinite dimensional density} is based on~\citep[Chap. 10.3]{da2002second}. Since $Q_{\infty} = -\frac{1}{2}Q \mc{A}^{-1}$ is a trace class, we define a trace class operator $\Theta_t$ as follows:
\begin{equation}
    \Theta_t = Q_{\infty}^{1/2} (Q^{-1/2}_t e^{t\mc{A}})^{*}(Q_{\infty}^{-1/2} Q^{1/2}_t)^{*} (Q_{\infty}^{1/2} (Q^{-1/2}_t e^{t\mc{A}})^{*}(Q_{\infty}^{-1/2} Q^{1/2}_t)^{*})^{*},
\end{equation}
for all $t \geq 0$. Since $Q_t = Q_{\infty } - e^{t\mc{A}}Q_{\infty}e^{t\mc{A}^{*}}$, we can rewrite $Q_t$ in terms of $\Theta_t$,
\begin{align}
    Q_t &= Q_{\infty} - e^{t\mc{A}}Q_{\infty}e^{t\mc{A}^{*}}\\
    & = Q^{1/2}_{\infty}\left[1 - (Q^{-1/2}_{\infty} e^{t\mc{A}}) Q_{\infty}(Q^{-1/2}_{\infty} e^{t\mc{A}})^{*}\right]Q^{1/2}_{\infty} \\
    & = Q^{1/2}_{\infty}(1-\Theta_t)Q^{1/2}_{\infty}.
\end{align}
Thus, we have $(1-\Theta_t)\mb{x} = Q^{-1/2}_{\infty}Q_t Q^{-1/2}_{\infty} \mb{x}$ for all $\mb{x} \in \mc{H}_0$. It implies that $\la (1-\Theta_t)\mb{x}, \mb{x}\ra_{\mc{H}_0} \geq 0$, the non-negativity of $(1-\Theta_t)$. Moreover it also implies that $(1-\Theta_t)^{-1}$ is invertible:
\begin{equation}\label{eq:one minus theta_t}
    (1-\Theta_t)^{-1} = (Q^{-1/2}_t Q^{1/2}_{\infty})^{*} Q^{-1/2}_t Q^{1/2}_{\infty}.
\end{equation}
Consequently, it yields the following formula~\citep[Proposition. 1.3.11]{da2002second}
\begin{align}\label{eq:proposition1.3.11}
    q_t(0, \mb{y}) = \textit{det}(1-\Theta)^{-1/2} \exp\left[-\frac{1}{2}\la \Theta_t (1-\Theta_t)^{-1} Q^{-1/2}_{\infty}\mb{y}, Q^{-1/2}_{\infty}\mb{y}\ra_{\mc{H}}\right].
\end{align}
Now, for the general case, by using the chain rule, we have:
\begin{equation}\label{eq:explicit density}
    q_t(\mb{x}, \mb{y}) = \frac{d\mc{N}_{e^{t\mc{A}}\mb{x}, Q_t}}{d\mc{N}_{0, Q_t}} \frac{d\mc{N}_{0, Q_t}} {d\mc{N}_{0, Q_{\infty}}}(\mb{y}) = \frac{d\mc{N}_{e^{t\mc{A}}\mb{x}, Q_t}}{d\mc{N}_{0, Q_t}}(\mb{y}) q_t(0, \mb{y}),
\end{equation}
and utilizing Cameron-Martin theorem~\citep[Theorem. 1.3.6]{da2002second}, we get:
\begin{align}
    & \frac{d\mc{N}_{e^{t\mc{A}}\mb{x}, Q_t}}{d\mc{N}_{0, Q_t}}(\mb{y}) = \exp \left[ \la  Q^{-1/2}_t e^{t\mc{A}}\mb{x}, Q^{-1/2}_t \mb{y} \ra_{\mc{H}} - \frac{1}{2}\norm{Q^{-1/2}_t e^{t\mc{A}}\mb{x}}^2_{\mc{H}}\right] \\
    & = \exp \left[ \la  Q^{-1/2}_t e^{t\mc{A}}\mb{x}, Q^{-1/2}_t \mb{y} \ra_{\mc{H}} - \frac{1}{2} \la Q^{-1/2}_t e^{t\mc{A}}\mb{x}, Q^{-1/2}_t e^{t\mc{A}}\mb{x} \ra_{\mc{H}}\right] \\
    & = \exp \left[ \la  Q_{\infty}^{1/2} Q^{-1}_t e^{t\mc{A}}\mb{x}, Q_{\infty}^{-1/2} \mb{y} \ra_{\mc{H}} - \frac{1}{2} \la Q^{1/2}_{\infty}Q^{-1}_t e^{t\mc{A}}\mb{x}, Q^{-1/2}_{\infty} e^{t\mc{A}}\mb{x} \ra_{\mc{H}}\right] \\
    & \stackrel{(i)}{=} \exp \left[ \la (1-\Theta_t)^{-1} Q^{-1/2}_{\infty} e^{t\mc{A}}\mb{x}, Q^{-1/2}_{\infty} \mb{y} \ra_{\mc{H}} - \frac{1}{2} \la (1-\Theta_t)^{-1} Q^{-1/2}_{\infty} e^{t\mc{A}}\mb{x},Q^{-1/2}_{\infty}  e^{t\mc{A}}\mb{x} \ra_{\mc{H}}\right] \label{eq:cameron-martin theorem}
\end{align}
where $(i)$ follows from~\eqref{eq:one minus theta_t}, $(1-\Theta_t)^{-1} Q^{-1/2}_{\infty} = (Q^{-1/2}_t Q^{1/2}_{\infty})^{*} Q^{-1/2}_t = Q_{\infty}^{1/2} Q^{-1}_t$. Thus, by substituting \eqref{eq:proposition1.3.11} and \eqref{eq:cameron-martin theorem} into \eqref{eq:explicit density}, we obtain the following result:
\begin{align}
     q_t(\mb{x}, \mb{y}) &= \textit{det}(1-\Theta)^{-1/2} \exp\bigg[-\frac{1}{2}\la \Theta_t (1-\Theta_t)^{-1} Q^{-1/2}_{\infty}\mb{y}, Q^{-1/2}_{\infty}\mb{y}\ra_{\mc{H}} \\
     & + \la (1-\Theta_t)^{-1} Q^{-1/2}_{\infty} e^{t\mc{A}}\mb{x}, Q^{-1/2}_{\infty} \mb{y} \ra_{\mc{H}} - \frac{1}{2} \la (1-\Theta_t)^{-1} Q^{-1/2}_{\infty} e^{t\mc{A}}\mb{x},Q^{-1/2}_{\infty}  e^{t\mc{A}}\mb{x} \ra_{\mc{H}} \bigg].
\end{align}
It concludes the proof.
\end{proof}

\subsection{Derivation of Example~\ref{example:diffusion bridge}}
For a diffusion bridge process, let us define the $h$ function as:
\begin{equation}
    h(t, T, \mb{x}_t, \mb{x}_T) =  \mathbb{E}_{\mathbb{P}}\left[\tilde{G}(\mb{X}_T, \mb{x}_T)|\mb{X}_t=\mb{x}_t\right] = \int \tilde{G}(\mb{z}, \mb{x}_T) d\mc{N}_{e^{(T-t)\mc{A}}\mb{x}_t, Q_{T-t}}(\mb{z})
\end{equation}

If we choose $\tilde{G}(\mb{x}, \mb{y}) = \mb{1}_{d\mb{y}}(\mb{x})$. Then, for any $\mb{y} \in \mc{H}$ and $t \in [0, T]$,  Theorem~\ref{Theorem:infinite dimensional density} implies that
\begin{align}
    h(0, t, \mb{x}_0, \mb{y}) &= \int \tilde{G}(\mb{z}, \mb{y}) d\mc{N}_{e^{t\mc{A}}\mb{x}_0, Q_t}(\mb{z}) \\
    & = \int \tilde{G}(\mb{z}, \mb{y}) \frac{d\mc{N}_{e^{t\mc{A}}\mb{x}_0, Q_t}}{d\mc{N}_{0, Q_{\infty}}}(\mb{z})  d\mc{N}_{0, Q_{\infty}}(\mb{z}) \\
    & = \int \tilde{G}(\mb{z}, \mb{y}) q_t(\mb{x}, \mb{z})  d\mc{N}_{0, Q_{\infty}}(\mb{z})  \\
    & = q_t(\mb{x}, \mb{y}).
\end{align}
Moreover, with an eigen-system of $\mc{H}$, $\{(\lambda^{(k)}, \phi^{(k)}) \in \mathbb{R} \times \mc{H} : k \in \mathbb{N}\}$, $q_t$ can be represented as~\citep{doi:10.1080/07362999308809319}:
\begin{align}\label{eq:orthonormal_1}
    q_t(\mb{x}, \mb{y}) = \prod_{k \in \mathbb{N}} q_t^{(k)} (\mb{x}^{(k)}, \mb{y}^{(k)}),
\end{align}
where for each coordinated $k$, $q_t^{(k)} (\mb{x}^{(k)}, \mb{y}^{(k)})$ has following representation:
\begin{align}\label{eq:orthonormal_2}
    & q_t^{(k)} (\mb{x}^{(k)}, \mb{y}^{(k)}) = \left(\frac{\lambda^{(k)}}{2a_k}(1 - e^{-2a_k t})\right)^{-1/2} \exp \left[- \frac{(\mb{y}^{(k)} - e^{-a_k t}\mb{x}^{k)})^2}{2\lambda^{(k)}(1 - e^{-2a_k t})} + \frac{(\mb{y}^{(k)})^2}{2\lambda^{(k)}}\right] \\
    & \mc{A}\phi^{(k)} = -a_k \phi^{(k)}, \quad Q \phi^{(k)} = \lambda^{(k)}\phi^{(k)}, \quad \mb{x}^{(k)} = \la \mb{x}, \phi^{(k)}\ra_{\mc{H}}, \quad \mb{y}_k = \la \mb{y}, \phi^{(k)}\ra_{\mc{H}}.\label{eq:orthonormal_3}
\end{align}
Therefore, since $D_{\mb{x}} \log h(t, T, \mb{x}, \mb{x}_T) = D_\mb{x}  \log q_{T-t}(\mb{x}_t, \mb{x}_T)$, by projecting $D_\mb{x}  \log q_{T-t}(\mb{x}_t, \mb{x}_T)$ to each coordinate $\phi^{(k)}$, we obtain the following results:
\begin{equation}\label{eq:orthonormal_4}
    \frac{d}{d\mb{x}^{(k)}} \log q_{T-t}^{(k)}(\mb{x}_t^{(k)}, \mb{x}_T^{(k)}) = \frac{2a_k e^{-a_k (T-t)}}{\lambda^{(k)}(1 - e^{-2a_k (T-t)})}(\mb{x}^{(k)}_T - e^{-a_k (T-t)}\mb{x}^{(k)})
\end{equation}

\subsection{Deriving Divergence Between Path Measures}\label{sec:Deriving Divergence between Path Measures_appx}
Here we present an infinite-dimensional generalization of Girsanov's theorem~\citep[Theorem~10.14]{da2014stochastic}, which plays a crucial role in estimating the divergence between two path measures discussed in Sec~\ref{sec:Approximating path measures}. The theorem is formulated as follows:
\begin{theorem}[Girsanov's Theorem in $\mc{H}$] Let $\gamma$ be a $\mc{H}_0$-valued $\mc{F}_t$-predictable process such that
\begin{equation}
    \mathbb{P}\left(\int_0^T \norm{\gamma_s}_{\mc{H}_0}^2 ds < \infty \right) = 1, \quad \mathbb{E}\left[\exp\left( \int_0^t \la \gamma_s, d\mb{W}^Q_t\ra_{\mc{H}_0} - \frac{1}{2} \int_0^T  \norm{\gamma_s}_{\mc{H}_0}^2 dt\right)\right]=  1.
\end{equation}
Then the process $\mb{\tilde{W}}^Q_t = \mb{W}^Q_t - \int_0^T \gamma_s ds$ is a $Q$-Wiener process with respect to $\{\mc{F}_t\}_{t \geq 0}$ on the probability space $(\Omega, \mc{F}, \mathbb{Q})$ where
\begin{equation}\label{eq:Girsanov_1_appx}
    d\mathbb{Q} = \exp\left( \int_0^t \la \gamma_s, d\mb{W}^Q_t\ra_{\mc{H}_0} - \frac{1}{2} \int_0^T  \norm{\gamma_s}_{\mc{H}_0}^2 dt\right) d\mathbb{P}.
\end{equation}
Or we can derive alternative formulation, by substituting $\mb{W}^Q_t = \mb{\tilde{W}}^Q_t +  \int_0^T \gamma_s ds$ to~\eqref{eq:Girsanov_1_appx},
\begin{equation}\label{eq:Girsanov_2_appx}
    d\mathbb{Q} = \exp\left( \int_0^t \la \gamma_s, d\mb{\tilde{W}}^Q_t\ra_{\mc{H}_0} + \frac{1}{2} \int_0^T  \norm{\gamma_s}_{\mc{H}_0}^2 dt\right) d\mathbb{P}.
\end{equation}
\end{theorem}

Now, we will apply Girsanov's theorem to path measures related to~\eqref{eq:uncontrolled SDE} and~\eqref{eq:conditioned SDE}. Let $\mathbb{Q} := \mathbb{P}^{\alpha}$ in~\eqref{eq:Girsanov_2_appx}. Then $\gamma_s := Q^{1/2}\alpha_s$ and we get
\begin{equation}
    \frac{d\mathbb{P}^{\alpha}}{d\mathbb{P}} = \exp\left( \int_0^t \la \alpha, d\mb{\tilde{W}}^Q_t\ra_{\mc{H}_0} + \frac{1}{2} \int_0^T  \norm{Q^{1/2}\alpha_s}_{\mc{H}_0}^2 dt\right).
\end{equation}
\begin{proof}
    Proof can be founded in~\citep[Theorem~10.14]{da2014stochastic}
\end{proof}
Since our goal is find an optimal control $\alpha^{\star}$ such that $\mb{X}_T^{\alpha^{\star}}$ satisfying terminal constraints which is represented by function $\tilde{G}$ in~\eqref{eq:h function}, we may define our target path measure $\mathbb{P}^{\star}$ as $\frac{d\mathbb{P}^{\star}}{d\mathbb{P}}  = \frac{1}{\mc{Z}}\tilde{G}(\cdot) $~\citep{richter2024improved}, where $\mc{Z} = \mathbb{E}_{\mathbb{P}}\left[\tilde{G}(\mb{X}_T)\right]$ Then, we can compute the logarithm of Radon Nikodym derivative as
\begin{align}
    \log \frac{d\mathbb{P}^{\alpha}}{d\mathbb{P}^{\star}} &= \log \frac{d\mathbb{P}^{\alpha}}{d\mathbb{P}} + \log \frac{d\mathbb{P}}{d\mathbb{P}^{\star}} - \mc{Z} \\
    & \approx \int_0^t \la \alpha, d\mb{\tilde{W}}^Q_t\ra_{\mc{H}_0} + \frac{1}{2} \int_0^T  \norm{Q^{1/2}\alpha_s}_{\mc{H}_0}^2 dt + G(\cdot).
\end{align}
Since $\mb{\tilde{W}}^Q_t$ is $Q$-Wiener process on $\mathbb{P}^{\alpha}$, we can compute the relative entropy loss in Sec~\ref{sec:Approximating path measures}:
\begin{equation}\label{eq:relative entropy loss_appx}
    D_{\text{ref}}(\mathbb{P}^{\alpha}|\mathbb{P}^{\star}) = \mathbb{E}_{\mathbb{P}^{\alpha}}\left[\log \frac{d\mathbb{P}^{\alpha}}{d\mathbb{P}^{\star}}\right] = \mathbb{E}_{\mathbb{P}^{\alpha}}\left[\frac{1}{2} \int_0^T  \norm{\alpha_s}_{\mc{H}}^2 dt + G(\mb{X}_T^{\alpha}) \right],
\end{equation}
where we denote $\norm{(\cdot)}_{\mc{H}_0}^2 = \norm{Q^{-1/2}(\cdot)}_{\mc{H}}^2 $
This representation matches with~\eqref{eq:objective functional} when $R(\cdot):=\frac{1}{2}\norm{\cdot}_{\mc{H}}^2$. Similarly, the cross entropy loss in~\eqref{eq:cross entropy loss} can be derived by set $\mathbb{Q}:=\mathbb{P}^{\star}$ and $\mathbb{P} := \mathbb{P}^{\alpha}$ in~\eqref{eq:Girsanov_2_appx}. Then $\gamma_s$ can be defined as~\eqref{eq:cross entropy gamma} and $\mb{\hat{W}}^Q_t = \mb{\tilde{W}}^Q_t + \int_0^T \gamma_s(\theta) ds$ is a $Q$-Wiener process on $\mathbb{P}^{\star}$ where $\mathbb{P}^{\alpha}$ satisfies the Radon-Nikodym derivative:
\begin{equation}
    \frac{d\mathbb{P}^{\star}}{d\mathbb{P}^{\alpha}} = \exp\left( \int_0^t \la \gamma_s(\theta), d\mb{\hat{W}}^Q_t\ra_{\mc{H}_0} + \frac{1}{2} \int_0^T  \norm{\gamma_s(\theta)}_{\mc{H}_0}^2 dt\right).
\end{equation}
Hence, the cross-entropy loss can be computed as in~\eqref{eq:cross entropy loss}.

\subsection{Proof of Theorem~\ref{theorem:mixture of bridge}}
\begin{proof}
Let us consider that the marginal distributions $\{\mu^{\star}_t\}_{t \in 0, T}$ satisfying the following relation in a weak sense~\citep{belopolskaya2012stochastic}:
\begin{equation}
    \partial_t \int f(\mb{x}_t) \mu^{\star}_t(d\mb{x}_t) =  \int f(\mb{x}_t)  \mc{L}_t^{\star}  \mu^{\star}_t(d\mb{x}_t) = \int \mc{L}_t f(\mb{x}_t) \mu^{\star}_t(d \mb{x}_t)
\end{equation}
where $\mc{L}^{\star}_t$ is adjoint operator of $\mc{L}_t$. Now, let us denote $\mu_{t|0, T}(\mb{x}_t) = \mu^{\star}_t(\mb{x}_t|\mb{x}_0, \mb{x}_T)$ and consider factorizable marginal distribution $\mu^{\star}_t(d\mb{x}_t) = \int_{\Pi} \mu^{\star}_t(\mb{x}_t|\mb{x}_0, \mb{x}_T) \Pi(d\mb{x}_0, d\mb{x}_T)$, where $\mu^{\star}_t(\mb{x}_t|\mb{x}_0, \mb{x}_T)$ satisfying the following relation:
\begin{align*}
    & \partial_t \int_{\mc{H}} f(\mb{x}_t) \mu^{\star}_t(d\mb{x}_t|\mb{x}_0, \mb{x}_T) = \int_{\mc{H}} \mc{L}_{t|0,T}f(\mb{x}_t) \mu^{\star}_t(d\mb{x}_t|\mb{x}_0, \mb{x}_T)   \\
    & = \int_{\mc{H}} \left[ \la \mc{A}\mb{x}_t, D_{\mb{x}} f(\mb{x}_t)\ra_{\mc{H}} +  \la h_{t|0, T}(\mb{x}_t) ,D_{\mb{x}} f(\mb{x}_t)\ra_{\mc{H}} + \frac{1}{2}\text{Tr}\left[ \sigma^2 Q D_{\mb{xx}}f(\mb{x}_t)\right] \right]\mu^{\star}_t(d\mb{x}_t|\mb{x}_0, \mb{x}_T),
\end{align*}
where we denote $h_{t|0, T}(\mb{x}_t) := \sigma^2 Q D_{\mb{x}} \log h(t, \mb{x_t})$.
Now, we can obtain the Kolmogorov operator associated with the diffusion process associated with the marginal distributions $\mu^{\star}_t(d\mb{x}_t)$ as follows
\begin{align}
    &\partial_t \int_{\mc{H}} f(\mb{x}_t) \mu^{\star}_{t}(d\mb{x}_t) = \partial_t 
  \int_{\Pi}  \int_{\mc{H}} f(\mb{x}_t) \mu^{\star}_t(d\mb{x}_t|\mb{x}_0, \mb{x}_T) \Pi(d\mb{x}_0, d\mb{x}_T) \\
 & = \int_{\mc{H}} \left[ \la \mc{A}\mb{x}_t, D_{\mb{x}} f(\mb{x}_t)\ra_{\mc{H}} +  \la h^{\star}_t(\mb{x}_t) ,D_{\mb{x}} f(\mb{x}_t)\ra_{\mc{H}} + \frac{1}{2}\text{Tr}\left[ \sigma^2 Q D_{\mb{xx}}f(\mb{x}_t)\right] \right] \mu^{\star}_t(d\mb{x}_t) \\
 & = \int_{\mc{H}} \mc{L}_t f(\mb{x}_t) \mu^{\star}_{t}(d\mb{x}_t) = \int f(\mb{x}_t)  \mc{L}_t^{\star}  \mu^{\star}_t(d\mb{x}_t),
\end{align}
where we have defined $\int_{\mc{H}} h^{\star}_t(\mb{x}_t) \mu^{\star}_t(d\mb{x}_t) :=  \int_{\Pi} \int_{\mc{H}} h_{t|0, T} \mu^{\star}_t(d\mb{x}_t|\mb{x}_0, \mb{x}_T)\Pi(d\mb{x}_0, d\mb{x}_T)$. It implies that the diffusion dynamics $d\mb{X}_t^{\star}$ associated with the Kolmogorov operator $\mc{L} :=\la \mc{A}\mb{x}_t, D_{\mb{x}} f(\mb{x}_t)\ra_{\mc{H}} +  \la h^{\star
 }_t(\mb{x}_t) ,D_{\mb{x}} f(\mb{x}_t)\ra_{\mc{H}} + \frac{1}{2}\text{Tr}\left[ \sigma^2 Q D_{\mb{xx}}f(\mb{x}_t)\right]$ also associated with Fokker-Planck equation~\citep{bogachev2011uniqueness} $\mc{L}_t^{*}  \mu^{\star}_t(d\mb{x}_t)=0$, meaning $\mb{X}_t^{\star} \sim \mu^{\star}_t$ for all $t \in [0, T]$. Now, for some reference measure $\mu_{\text{ref}}$ 
 where the Radon-Nikodym derivatives $\frac{d\mu^{\star}_t}{d\mu_{\text{ref}}}(\mb{x}_t) = p_t(\mb{x}_t), \frac{d\mu_{t|0, T}}{d\mu_{\text{ref}}}(\mb{x}_t) = p_{t|0, T}(\mb{x}_t)$ exist. Then, under $\mu_{\text{ref}}$, $h^{\star}_t$ can be defined as follows:
\begin{align}
        h^{\star}_t(\mb{x}_t) 
        &=  \frac{\int_{\Pi} h_{t|0, T}(\mb{x}_t) p_{t|0, T}(\mb{x}_t)   \Pi(d\mb{x}_0, d\mb{x}_T)}{p_t(\mb{x}_t)}.
\end{align}
Therefore, the diffusion process associated with marginal distributions $\mu^{\star}_t(\mb{x}_t)$ has following representation:
\begin{align}
    d\mb{X}^{\star}_t = \left[\mc{A}\mb{X}^{\star}_t + \sigma^2 Q \mathbb{E}_{\mb{x}_T \sim \mathbb{P}^h(d\mb{x}_T|\mb{X}^{h}_t)}\left[q_{T-t}(\mb{X}^{\star}_t, \mb{x}_T)\right] \right]dt + \sigma d\mb{W}^Q_t, \quad \mb{X}^{\star}_0 \sim \mu^{\star}_0.
\end{align}
This concludes the proof.
\end{proof}
\subsection{Proof of Theorem~\ref{theorem:exact sampling}}
\begin{proof} Let us assume the initial condition is fixed to deterministic point $\mb{x}_0 \in \mc{H}$ and define a reference measure $\mu_{\text{ref}} = \mc{N}(0, Q_{\infty})$. Since our goal is $\mb{X}^{\alpha}_T$ satisfying the terminal condition $G$, define $\mathbb{P}^{\star}$ as $\frac{d\mathbb{P}^{\star}}{d\mathbb{P}}  = \frac{1}{\mc{Z}}\tilde{G}(\cdot) $, where $\mc{Z} = \mathbb{E}_{\mathbb{P}}[\tilde{G}(\mb{X}_T)]$. Therefore, we have the following relation for marginal distributions $d\mu_T^{\star} = \frac{1}{Z}\tilde{G}(\mb{X}_T) d\mu_T$.
Moreover, since we have defined $G(\mb{X}_T) = -\log \frac{d\pi_T}{d\mu_T}(\mb{X}_T)$ in~\eqref{eq:objective function exact sampling}, then it result $h(t, \mb{x}) = \mathbb{E}_{\mathbb{P}}\left[\frac{d\pi_T}{d\mu_T}(\mb{X}_T) | \mb{X}_t=\mb{x}\right]$ by following Theorem~\ref{Theorem:Hopf-Cole}, we get $\tilde{G}(\mb{X}_T) = h(T, \mb{X}_T)$. Hence, we have for any Borel set $B \in \mc{B}(\mc{H})$,
\begin{align}
    \int_B d\mu^{\star}_T(v) &= \int_B \frac{h(T, \mb{X}_T)}{h(0, \mb{x}_0)} d\mu_T \\
    & = \int_B \frac{\frac{d\pi_T}{d\mu_T}(\mb{X}_T)}{\int_{\mc{H}} \frac{d\pi_T}{d\mu_T}(\mb{X}_T) d\mu_T} d\mu_T = \int_B d\pi_T
\end{align}
Now, given that we have confirmed that the conditioned SDE in~\eqref{eq:conditioned SDE} correspond to the controlled process~\eqref{eq:conditioned SDE} with optimal control, we can establish the result $\mathbb{P}(\mb{X}^{\alpha^{\star}}_T \in B) = \pi(B)$. This concludes the proof.
\end{proof}

\subsection{Experimental Details}\label{sec:experimental details}
\subsubsection{Experiment on 2D-Domain}
\paragraph{Synthetic Experiment. } In a synthetic experiment, we computed the log-probability of $p_0$ and $p_T$ across a uniformly sampled grid of $64^2$ points, with each point $\mb{p}_i$ ranging within $[-7, 7]^2$. For $p_0$, the log-probability was generated using an 8-Gaussian mixture model as specified in~\citep{tong2024improving}. For $p_T$, we employed the following log-density function:
\begin{equation}
\log p_T(\mb{p}) = -\frac{\min(\norm{\mb{p} - 1}^2, \norm{\mb{p} - 3}^2, \norm{\mb{p} - 5}^2)}{0.05}.
\end{equation}

Training was conducted using the Adam optimizer with a learning rate of $1e-3$. The network was trained with a batch size of 24 for a total of 1000 iterations. We set $\sigma=0.2$ in~\eqref{eq:uncontrolled SDE} for this experiment and set 100 discretization steps. We use a single A$6000$ GPU for this experiment. 

The control function was parameterized using a $4$-layer FNO-2D~\citep{li2020fourier}, with the cutoff number of Fourier modes set at $8$ and each convolution layer having a width of $32$.

\begin{figure}[!t]
\includegraphics[width=1.\textwidth,]{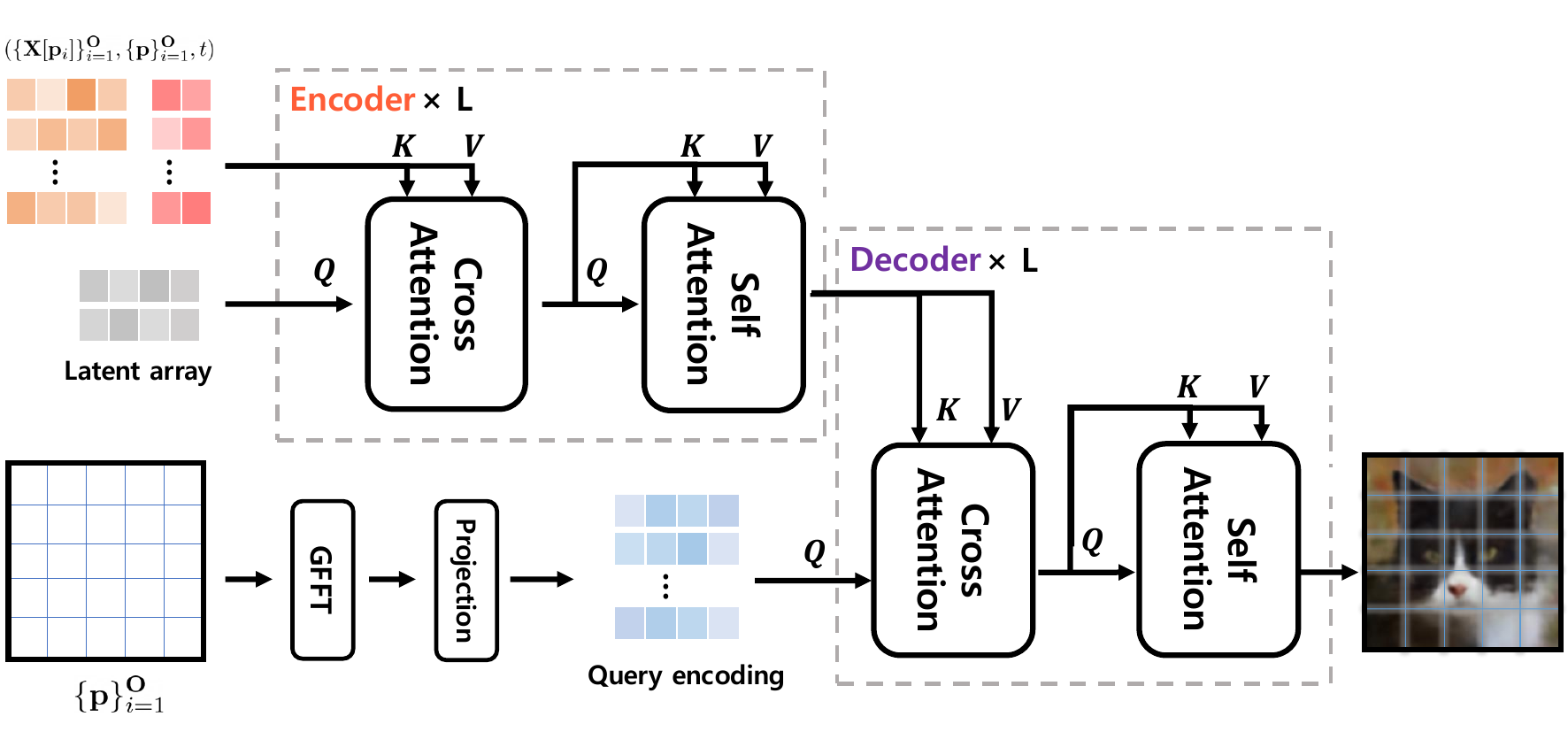}
\caption{Transformer-based network architecture.}\label{figure:architecture}
\end{figure}

\paragraph{Simulation of DBFS.} We follow the simulation scheme introduced in~\citep{rissanen2023generative,lim2023scorebased}. For $\mb{x} = (x_1, x_2) \in \mathbb{R}^2$, we use the discrete cosine transformation (\texttt{DCT}) for projection. Specifically, the eigenvector $\phi^{(k)}(\mb{x})$  and eigenvalue $\lambda^{(k)}$ of the negative Laplacian operator $-\Delta$, which is positive definite and Hermitian, and satisfies the zero Neumann boundary condition, are given by:
\begin{align}
    & -\Delta \phi^{(k)}(\mb{x}) = \lambda^{(k)} \phi^{(k)}(\mb{x}) \\
    & \frac{\partial \phi^{(k)}(\mb{x})}{\partial x_1} = \frac{\partial \phi^{(k)}(\mb{x})}{\partial x_2} = 0,.
\end{align}
It implies that the orthonormality of $\phi^{(k)}(\cdot)$ with respect to the associated inner product, thereby enables computations in (\ref{eq:orthonormal_1}-\ref{eq:orthonormal_4}). Now, considering a rectangular domain with Cartesian coordinates, where pixels in the image are sampled from an underlying regular domain, the eigenbasis is given as a separable cosine basis:
\begin{align}
    & \phi^{(n, m)}(x_1, x_2) \sim \cos{\left(\frac{\pi n x_1}{W}\right)}\cos{\left(\frac{\pi m x_2}{H}\right)} \\
    & \lambda^{(n, m)} = \pi^2 \left(\frac{n^2}{W^2} + \frac{m^2}{H^2} \right).
\end{align}
Then, the negative Laplacian $-\Delta$ can then by represented by an eigen decomposition $-\Delta = \mb{E}\mb{D}\mb{E}^T$, where $\mb{E}^T$ is the projection matrix for the \texttt{DCT} $\ie \tilde{\mb{X}}_t = \mb{E}^T \mb{X}_t =\texttt{DCT}(\mb{X}_t)$, and $\mb{D}$ is a digonal matrix containing the eigenvalues $\lambda^{(n, m)}$. Hence the controlled SDEs~\eqref{eq:controlled SDE} with $Q=\mb{I}$  can be rewritten as
\begin{equation}
    d\tilde{\mb{X}}^{\alpha}_t = \left[ -\mb{D}\tilde{\mb{X}}^{\alpha}_t  + \sigma \tilde{\alpha_t} \right] dt + \sigma d\tilde{\mb{W}}_t, \quad t \in [0, T],
\end{equation}
where $\tilde{\mb{X}}_t = \mb{E}^T \mb{X}_t$,$\tilde{\alpha}_t = \mb{E}^T \alpha_t$, and $\tilde{\mb{W}}_t \stackrel{d}= \mb{W}_t$ for all $t \in [0, T]$.

\paragraph{Sampling Algorithm. }
The sampling algorithm for DBFS in bridge matching, discussed in Section~\ref{sec:bridge matching experiment}, is provided in detail in Algorithm~\ref{algorithm:DBFS sampling}.

 \begin{algorithm}[t!]
    \caption{DBFS sampling for Bridge Matching}
    \begin{algorithmic}\label{algorithm:DBFS sampling}
    \STATE \textbf{Input:} Linear operator $\mc{A}$, trace-class operator $Q$, learned control $\alpha$, initial distribution $\pi_0$, discrete and inverse discrete cosine transforms $\texttt{DCT}, \texttt{iDCT}$, target resolution grid $\mb{T}^2$, trained resolution grid $\mb{O}^2$.
    \\\hrulefill
    \STATE Sample the initial condition $\mb{X}^{\alpha}_0 = \mb{x}_0 \sim \pi_0$
        \IF{$\mb{T}^2 \neq \mb{O}^2$}
        \STATE Upsample $\mb{X}^{\alpha}_0 = \{\mb{X}^{\alpha}_0[\mb{p}_i]\}_{i=1}^{\mb{O}^2}$ to $\{\mb{X}^{\alpha}_0[\mb{p}_i]\}_{i=1}^{\mb{T}^2}$
        \ENDIF
      \FOR{$t=0, \cdots, T$}
        \STATE Estimate the control $\alpha^{\star}_t = \alpha(t, \mb{X}^{\alpha}_t;\theta^{\star})$
        \STATE Sample Gaussian noise $\xi \sim \mc{N}(0, \mb{I})$
        \STATE Discrete cosine transform the initial condition, estimated control, Gaussian noise
        \STATE $\{\tilde{\mb{X}}_t^{(k)}\}_{k=1}^{\mb{T}^2} = \texttt{DCT}(\mb{X}^{\alpha}_t), \quad \{\tilde{\alpha}^{(k)}_t\}_{k=1}^{\mb{T}^2} = \texttt{DCT}(\alpha^{\star}_t), \quad \{\tilde{\xi}^{(k)}\}_{k=1}^{\mb{T}^2} = \texttt{DCT}(\xi)$
        \STATE \textbf{for} $k=1, \cdots, ^{\mb{T}^2}$ \textbf{do in parallel}
        \STATE\hspace{+4mm} $\tilde{\mb{X}}_{t+\Delta_t}^{(k)}= \left[- a_k  \tilde{\mb{X}}_{t}^{(k)} + \sigma \sqrt{\lambda_k} \tilde{\alpha}^{(k)}_t\right] \Delta_t + \sigma \sqrt{\lambda_k \Delta_t}  \tilde{\xi}^{(k)} $ 
        \STATE \textbf{end for}
        \STATE Inverse discrete cosine transform, $\mb{X}^{\alpha}_{t+\Delta_t} = \texttt{iDCT}(\{\tilde{\mb{X}}_{t+\Delta_t}^{(k)}\}_{k=1}^{\mb{T}^2})$
        \STATE $\mb{X}^{\alpha}_t = \mb{X}^{\alpha}_{t+\Delta_t}$
      \ENDFOR
    \STATE \textbf{Output:} $\mb{X}^{\alpha}_T \sim \pi_T$
    \end{algorithmic}
  \end{algorithm}

\paragraph{Unpaired dataset Transfer Experiment. } 

For the experiment involving transfer between dataset, we followed the setup described in~\citep{peluchetti2023diffusion}. 

\textbf{(A)} For the EMNIST and MNIST datasets, the initial distribution, $\pi_0$, was set as the MNIST dataset, while for the terminal distribution, $\pi_T$, we used the EMNIST dataset with the first five lowercase and uppercase characters, as outlined by~\citep{de2021diffusion}. The iterative training scheme proposed by~\citep{peluchetti2023diffusion}. was adopted, which involved two neural networks, $\alpha_t(t, \mathbf{x}, \theta)$ and $\alpha_t(t, \mathbf{x}, \psi)$, each with around 20.7 million parameters. These networks approximate mixtures of bridges for the forward ($\pi_0 \to \pi_T$) and reverse ($\pi_T \to \pi_0$) directions, respectively. The SDE was discretized into 30 steps without a noise schedule. The DBFS model was trained for 60 iterations, with each iteration comprising 5,000 gradient updates. Additionally, 2,560 cached images were used for training each network, updated every 250 steps. We used the Adam optimizer with a learning rate of 1e-4 and a batch size of 128, with the EMA rate set to 0.999. The complete DBFS training for the MNIST experiment took approximately 15 hours on a single A6000 GPU.

\textbf{(B)} For the AFHQ dataset~\cite{choi2020starganv2}\footnote{\url{https://github.com/clovaai/stargan-v2}, under CC BY-NC 4.0 License. }, we evaluated DBFS between the \texttt{wild} and \texttt{cat} classes on a $64^2$ grid, with each class containing approximately 5,000 samples. The control networks each contained about 120.3 million parameters. We discretized the SDE into 100 steps without using a noise schedule. The Adam optimizer was used with a learning rate of 1e-4, and the EMA rate was set to 0.999. We used a batch size of 64 and trained for 20 iterations, with a total of 400,000 gradient steps. Additionally, 2,560 cached images were used for training each network, updated every 1,000 steps. The training took approximately 8 days, using 8 A6000 GPUs for the experiment.

\begin{table}[h!]
\caption{Network Hyper-parameters}
\label{table:training_hyperparameters}
\centering
 \resizebox{\textwidth}{!}{
\begin{tabular}{c|ccccccc}
\toprule
\textbf{Dataset} & \textbf{Latent dim} & \textbf{Position dim} &\textbf{$\#$heads} & \textbf{$\#$enc blocks} & \textbf{$\#$dec blocks} & \textbf{$\#$self attn. per block} & \textbf{$\#$ of parameters} \\ 
\midrule
MNIST    & 256 & 256 &4 & 6 & 2 & 1 & 20.7M \\
AFHQ  & 512 & 512 &4 & 6 & 2 & 2 & 120.3M\\
\bottomrule
\end{tabular}}
\end{table}

For each control network $\alpha$, we used a transformer network architecture inspired by PerceiverIO ~\citep{jaegle2022perceiver} from public repository\footnote{\url{https://github.com/lucidrains/perceiver-pytorch}, under MIT License.} to model functional representation. This transformer architecture was chosen for its efficiency in evaluating field representations over a large number of grid points and its ability to map arbitrary input arrays to arbitrary output arrays in a domain-agnostic way. Additionally, we adopted the attention mechanism proposed in~\citep{peebles2023scalable}. The decoding cross-attention mechanism was also modified, inspired by~\citep{li2022transformer}, with the output query set to the target grid points, which were transformed as Gaussian Fourier features~\citep{NEURIPS2020_55053683}.

In practice, we start by passing the evaluations \(\mathbf{X}[\mathbf{p}]\) through a single MLP layer and combine it with the Fourier feature embeddings of the grid points \(\mathbf{p}\). This combined representation is then input into the encoder blocks as keys and values, where QKV attention is first applied with the latent array as the query, followed by self-attention for each block. We implement the time-modulated attention block proposed by~\citep{peebles2023scalable}, embedding the time \(t\) into latent space. Next, the target grid points, represented as Fourier feature embeddings of the grid points \(\mathbf{p}\), are fed into the decoder blocks as queries. Here, QKV attention is applied with the encoded latent array serving as keys and values, followed by self-attention in each decoder block. Finally, the decoded array is mapped to grey scale or RBF channels using a single linear layer. Conceptual illustrations of the proposed network are presented in Figure~\ref{figure:architecture}, and detailed network hyperparameters are listed in Table~\ref{table:training_hyperparameters}.

\begin{figure}[!h]
  \centering
  % \begin{minipage}[t]{.49\textwidth}
    \centering
    \subfigure[$32^2$ (observed resolution)]{
        \includegraphics[width=0.315\textwidth]{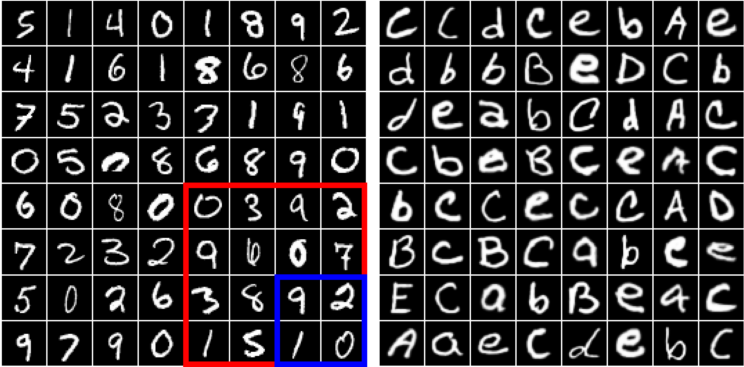}
    }
    \subfigure[$64^2$ (unseen resolution)]{
        \includegraphics[width=0.315\textwidth]{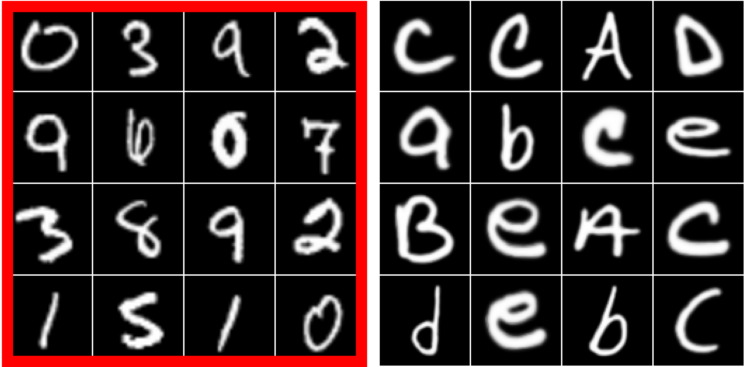}
    }
    \subfigure[$128^2$ (unseen resolution)]{
        \includegraphics[width=0.315\textwidth]{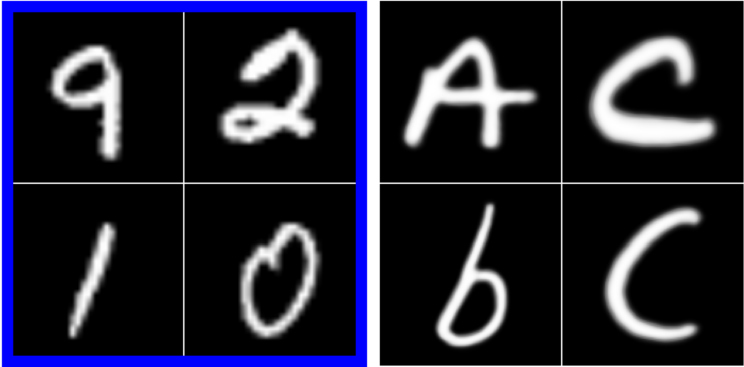}
    }
    \subfigure[$64^2$ (observed resolution)]{
        \includegraphics[width=0.475\textwidth]{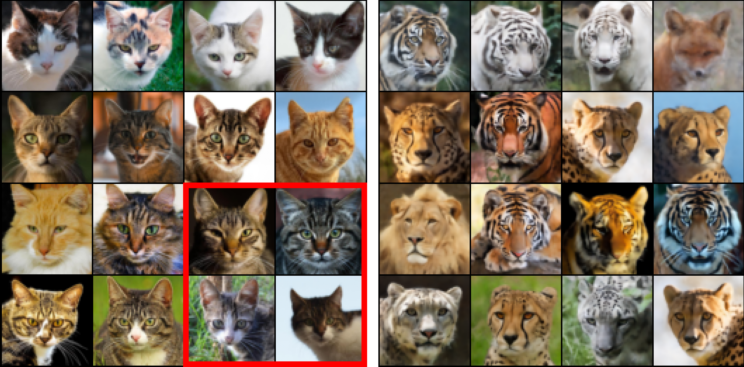}
    }
    \subfigure[$128^2$ (unseen resolution)]{
        \includegraphics[width=0.475\textwidth]{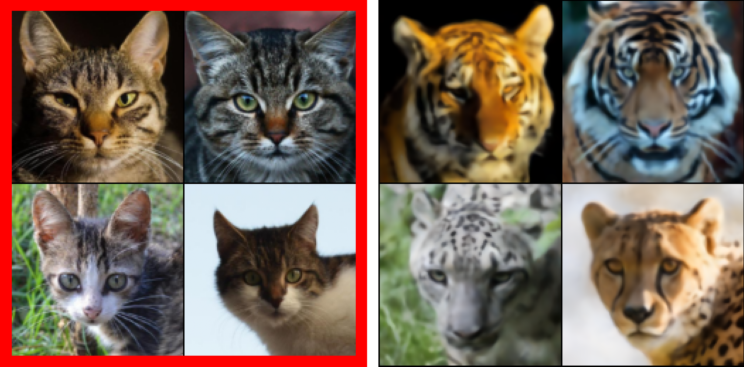}
    }
    \caption{Results on Unpaired image transfer task. \textbf{(Up)} MNIST $\to$ EMNIST \textbf{(Down)} AFHQ-64 Cat $\to$ Wild. (Left) Real data and (Right) generated samples from our model. For generation at unseen resolutions, the images within the \textcolor{red}{red} and \textcolor{blue}{blue} boxed initial conditions were upsampled (using bi-linear transformation) from the observed resolution ($32^2$) for EMNIST and ($64^2$) for AFHQ-64 Cat, respectively.}\label{figure:2d_generatior_afhq_}
  % \end{minipage}
\vspace{-5mm}
\end{figure}

\begin{figure}[!h]
  \centering
  % \begin{minipage}[t]{.49\textwidth}
    \centering
    \subfigure[$64^2$ (observed resolution)]{
        \includegraphics[width=0.475\textwidth]{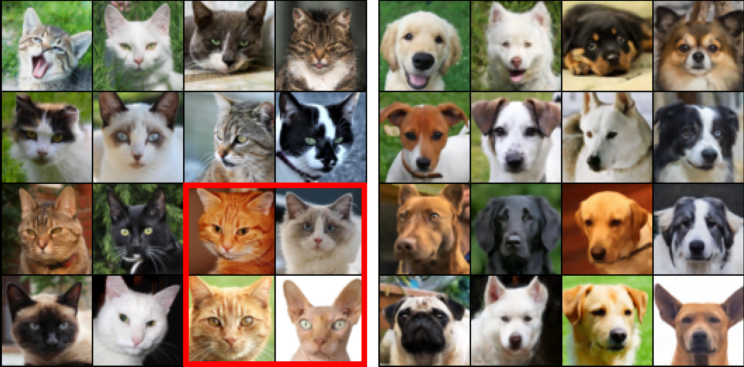}
    }
    \subfigure[$128^2$ (unseen resolution)]{
        \includegraphics[width=0.475\textwidth]{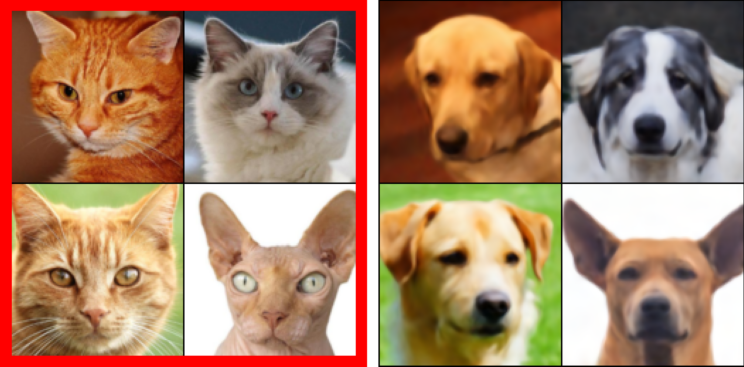}
    }
    \caption{Results on Unpaired image transfer task. AFHQ-64 Cat $\to$ Dog.}
  % \end{minipage}
\vspace{-5mm}
\end{figure}\label{figure:2d_generatior_afhq__}

\subsubsection{Experiment on 1D-Domain}\label{sec:modeling termpoal function_appx}

For the experiments on 1D-domain, we consistently set $\mc{A} := -\frac{1}{2}$ and $Q :=  \exp(-\norm{\mb{p} - \mb{p}'}^2/\gamma)$ and set $\gamma=0.2$ and $\gamma=0.02$ for GP regression and imputaion, respectively. The choice of $\gamma$ is hyper-parameter, we search over the set $[0,01, 0.1, 0.2, 0.5, 1.0]$ and find optimal value for GP regression. For imputation, we set $\gamma=0.02$ by following~\citep{bilovs2023modeling}.
\paragraph{Terminal Cost Computation} For all experiments conducted in the 1D domain, we implemented a parameterized initial condition which takes as input the observed sequences $\mb{X}^{\theta}_0 = \mb{x}^{\theta}(\mb{Y}[\mb{O}])$. We employed the energy functional $\mc{U}$ as the Gaussian negative log-likelihood (NLL). For each evaluation point on $\mb{T}$, $\mc{U}$ can be computed as follows:
\begin{equation}
\mc{U}(\mb{X}_T[\mb{T}]) = -\log \mc{N}(\mb{X}_T[T] | \mb{Y}, \sigma_{\theta}) = \sum_{i=1}^{|\mb{T}|} \frac{\norm{\mb{X}_T[\mb{p}_i] - \mb{Y}[\mb{p}_i]}^2}{2 \sigma^2_{\theta}},
\end{equation}
where $\sigma^2_{\theta}$ is set as an output from the neural network in accordance with~\citep{lee2020bootstrapping} for GP regression, and fixed as $\sigma^2_{\theta}=0.5$ for imputation, to establish a loss function analogous to~\citep{bilovs2023modeling}. Additionally, we specified a learnable prior distribution $\mu_{\text{prior}} = \mc{N}(e^{-\frac{T}{2}\mb{x}^{\theta}}, Q_T)$. Consequently, the terminal cost retains only the NLL term, simplifying the computation.
\paragraph{GP regression} For the GP regression, we borrow the experiment setting from~\citep{lee2020bootstrapping}. The model trained with curves generated from GP with RBF kernel and tested in various settings such as data generated from GP with other type of kernel (Matérn 5/2, Periodic). We generated $\mb{p}$ uniformly on the interval $[-2, 2]$ and generated $\mb{Y}[\mb{p}]$ from using RBF kernel $\kappa(\mb{p}_i, \mb{p}_j) = l^2_1 \exp(-\norm{\mb{p}_i - \mb{p}_j}^2/l^2_2)$ with $l_1 \sim \text{Unif}(0.1, 1.0)$ and $l_2 \sim \text{Unif}(0.1, 0.6)$ and the white noise $\xi \sim \mc{N}(0, 1e-2)$ is added. We set $|\mb{O}|$ randomly from $\text{Unif}(3, 37)$ and $|\mb{T}|$ from $\text{Unif}(3, 50 - |\mb{O}|)$. For the other test data, we define $\kappa(\mb{p}_i, \mb{p}_j) = l^2_1 (1 + \sqrt{5}d/l_2 + 5d^2/(3l^2_2))exp(-\sqrt{5}d/l_2)$ with $d = (\norm{\mb{p}_i - \mb{p}_j})$, $l_1 \sim \text{Unif}(0.1, 1.0)$ and $l_2 \sim \text{Unif}(0.1, 0.6)$ for Matérn kernel and $\kappa(\mb{p}_i, \mb{p}_j) = l^2_1 \exp(-2 \sin^2(\pi\norm{\mb{p}_i - \mb{p}_j}^2/p)/l_2)$ with $l_1 \sim \text{Unif}(0.1, 1.0)$, $l_2 \sim \text{Unif}(0.1, 0.6)$ and $p \sim \text{Unif}(0.1, 0.5)$ for periodic kernel.

We set batch size of $100$ and trained for $100,000$ iterations. The Adam optimizer is used, the initial learning rate 5e-4 decayed with cosine annealing scheme. For testing, we evaluated the trained models using 3,000 batches, each consisting of 16 samples. We report the mean and standard deviation for five runs. We a single A$6000$ GPU for this experiment.

The architectures for NP~\citep{garnelo2018neural} and CNP~\citep{pmlr-v80-garnelo18a}, we use the same setting as described in~\citep{lee2020bootstrapping}\footnote{\url{https://github.com/juho-lee/bnp}, under MIT License.}. In our approach, we adapted the CNP architecture to incorporate a parameterized initial condition $\mb{x}^{\theta}$ (add one linear layer to output $\mb{x}^{\theta}$). The total number of parameters is similar across all three models. 

\paragraph{Physionet Imputation} The Physionet~\citep{Physionet} contains $4000$ clinical time series with $35$ variables for $48$ hours from intensive care unit. Following~\citep{CSDI}, we preprocess this datasets to hourly time-series which have $48$ time steps. Since the dataset already contains around $80\%$ of missingness, we randomly choose $10/50/90\%$ of observed values as test data.

In the imputation experiments, we employed the same experimental setup as~\citep{bilovs2023modeling}\footnote{\url{https://github.com/morganstanley/MSML/tree/main/papers/Stochastic_Process_Diffusion}, under Apache $2.0$ License.} which is slight modification of original CSDI model. In this setup, the Gaussian noise of the DDPM model was replaced with GP noise employing an RBF kernel. Additionally, we adjusted the measurement approach to record the actual elapsed time rather than rounding to the nearest hour, better capturing the inherent timing characteristics of the Physionet dataset. We use a single A$6000$ GPU for this experiment.

\newpage

\section*{NeurIPS Paper Checklist}
\begin{enumerate}
\item {\bf Claims}
    \item[] Question: Do the main claims made in the abstract and introduction accurately reflect the paper's contributions and scope?
    \item[] Answer: \answerYes{} % Replace by \answerYes{}, \answerNo{}, or \answerNA{}.
    \item[] Justification: Yes, the main claims made in the abstract and introduction are consistent with the scope of the paper.

\item {\bf Limitations}
    \item[] Question: Does the paper discuss the limitations of the work performed by the authors?
    \item[] Answer: \answerYes{} % Replace by \answerYes{}, \answerNo{}, or \answerNA{}.
    \item[] Justification: The limitations are discussed in the conclusion section.

\item {\bf Theory Assumptions and Proofs}
    \item[] Question: For each theoretical result, does the paper provide the full set of assumptions and a complete (and correct) proof?
    \item[] Answer: \answerYes{} % Replace by \answerYes{}, \answerNo{}, or \answerNA{}.
    \item[] Justification: Yes, we believe our proof and assumptions are both sufficient and correct.

    \item {\bf Experimental Result Reproducibility}
    \item[] Question: Does the paper fully disclose all the information needed to reproduce the main experimental results of the paper to the extent that it affects the main claims and/or conclusions of the paper (regardless of whether the code and data are provided or not)?
    \item[] Answer: \answerYes{} % Replace by \answerYes{}, \answerNo{}, or \answerNA{}.
    \item[] Justification: Yes, we believe we have provided all the necessary information to reproduce our results in Appendix~\ref{sec:experimental details}.

\item {\bf Open access to data and code}
    \item[] Question: Does the paper provide open access to the data and code, with sufficient instructions to faithfully reproduce the main experimental results, as described in supplemental material?
    \item[] Answer: \answerYes{} % Replace by \answerYes{}, \answerNo{}, or \answerNA{}.
    \item[] Justification: We have provided the codebase used for our experiments, particularly for unpaired image transfer.

\item {\bf Experimental Setting/Details}
    \item[] Question: Does the paper specify all the training and test details (e.g., data splits, hyperparameters, how they were chosen, type of optimizer, etc.) necessary to understand the results?
    \item[] Answer: \answerYes{} % Replace by \answerYes{}, \answerNo{}, or \answerNA{}.
    \item[] Justification: Yes, we have provided the experimental details.
    
\item {\bf Experiment Statistical Significance}
    \item[] Question: Does the paper report error bars suitably and correctly defined or other appropriate information about the statistical significance of the experiments?
    \item[] Answer: \answerYes{} % Replace by \answerYes{}, \answerNo{}, or \answerNA{}.
    \item[] Justification: Yes, when available, we conducted the experiments five times and reported the mean and standard deviations.

\item {\bf Experiments Compute Resources}
    \item[] Question: For each experiment, does the paper provide sufficient information on the computer resources (type of compute workers, memory, time of execution) needed to reproduce the experiments?
    \item[] Answer: \answerYes{} % Replace by \answerYes{}, \answerNo{}, or \answerNA{}.
    \item[] Justification: Yes, we provided the required resources in the experimental details section.
    
\item {\bf Code Of Ethics}
    \item[] Question: Does the research conducted in the paper conform, in every respect, with the NeurIPS Code of Ethics \url{https://neurips.cc/public/EthicsGuidelines}?
    \item[] Answer: \answerYes{} % Replace by \answerYes{}, \answerNo{}, or \answerNA{}.
    \item[] Justification: We support the NeurIPS Code of Ethics.

\item {\bf Broader Impacts}
    \item[] Question: Does the paper discuss both potential positive societal impacts and negative societal impacts of the work performed?
    \item[] Answer: \answerYes{} % Replace by \answerYes{}, \answerNo{}, or \answerNA{}.
    \item[] Justification: The broader impacts are discussed in the conclusion section.

\item {\bf Safeguards}
    \item[] Question: Does the paper describe safeguards that have been put in place for responsible release of data or models that have a high risk for misuse (e.g., pretrained language models, image generators, or scraped datasets)?
    \item[] Answer: \answerYes{} % Replace by \answerYes{}, \answerNo{}, or \answerNA{}.
    \item[] Justification: We adhere to ethical standards for using our model in generative AI.
    
\item {\bf Licenses for existing assets}
    \item[] Question: Are the creators or original owners of assets (e.g., code, data, models), used in the paper, properly credited and are the license and terms of use explicitly mentioned and properly respected?
    \item[] Answer: \answerYes{} % Replace by \answerYes{}, \answerNo{}, or \answerNA{}.
    \item[] Justification: Yes, the license and terms of use are noted.

\item {\bf New Assets}
    \item[] Question: Are new assets introduced in the paper well documented and is the documentation provided alongside the assets?
    \item[] Answer: \answerNA{} % Replace by \answerYes{}, \answerNo{}, or \answerNA{}.
    \item[] Justification: The paper does not release new assets.

\item {\bf Crowdsourcing and Research with Human Subjects}
    \item[] Question: For crowdsourcing experiments and research with human subjects, does the paper include the full text of instructions given to participants and screenshots, if applicable, as well as details about compensation (if any)? 
    \item[] Answer: \answerNA{} % Replace by \answerYes{}, \answerNo{}, or \answerNA{}.
    \item[] Justification: We do not involve crowdsourcing or research with human subjects.

\item {\bf Institutional Review Board (IRB) Approvals or Equivalent for Research with Human Subjects}
    \item[] Question: Does the paper describe potential risks incurred by study participants, whether such risks were disclosed to the subjects, and whether Institutional Review Board (IRB) approvals (or an equivalent approval/review based on the requirements of your country or institution) were obtained?
    \item[] Answer: \answerNA{} % Replace by \answerYes{}, \answerNo{}, or \answerNA{}.
    \item[] Justification: We do not involve crowdsourcing or research with human subjects.
\end{enumerate}

\end{document}